\documentclass{article}

\usepackage{times}
\usepackage{graphicx} 
\usepackage{subfigure} 

\usepackage{natbib}

\usepackage{algorithm}
\usepackage{algorithmic}

\usepackage[nohyperref,accepted]{icml2015}

\usepackage{url}
\usepackage{graphicx}
\usepackage{algorithm}
\usepackage{algorithmic}
\usepackage{bm}
\usepackage{amsfonts}
\usepackage{amsthm}
\usepackage{amsmath}

\newcommand{\nn}{\nonumber\\}
\newcommand{\tmmax}[2]{\max_{#1}^{(#2)}}
\newcommand{\dmmax}[2]{{\max_{#1}}^{(#2)}}

\usepackage{subfigure}
\newlength{\subfigwidth}
\newlength{\subfigcolsep}
\newtheorem{theorem}{Theorem}
\newtheorem{lemma}[theorem]{Lemma}
\newtheorem{fact}[theorem]{Fact}

\newcommand{\Prob}{\mathrm{Pr}}
\newcommand{\Expect}{\mathbb{E}}
\newcommand{\Indicator}{{\mathbf{1}}}
\newcommand{\Ind}{\Indicator}

\newcommand{\thetai}{\theta_{i}}

\newcommand{\Beta}{\mathrm{Beta}}

\newcommand{\hatmu}{\hat{\mu}}

\newcommand{\opts}{{[L]}}
\newcommand{\subopts}{{[K] \setminus [L]}}

\newcommand{\lisubopts}{{[K] \setminus ([L-1] \cup \{i\})}}

\newcommand{\Regret}{\mathrm{Reg}}

\newcommand{\selected}{{I(t)}}

\newcommand{\Bernoulli}{\mathrm{Bernoulli}}
\newcommand{\EA}{\mathcal{A}}
\newcommand{\EB}{\mathcal{B}}
\newcommand{\EC}{\mathcal{C}}
\newcommand{\ED}{\mathcal{D}}
\newcommand{\EE}{\mathcal{E}}

\newcommand{\ES}{\mathcal{S}}
\newcommand{\ET}{\mathcal{T}}
\newcommand{\EU}{\mathcal{U}}

\newcommand{\thetaijss}{\theta_{\backslash i,j}^{**}}
\newcommand{\thetaijsst}{\theta_{\backslash i,j}^{**}(t)}
\newcommand{\epsilononed}{\delta}
\newcommand{\epsilonOne}{\epsilon_1}
\newcommand{\tilmu}{\theta}

\newcommand{\tilmuijsst}{\tilmu_{\backslash i,j}^{**}(t)}
\newcommand{\e}{\mathrm{e}}
\allowdisplaybreaks[1]

\icmltitlerunning{Optimal Regret Analysis of Thompson Sampling in Stochastic Multi-armed Bandit Problem with Multiple Plays}

\begin{document} 

\twocolumn[
\icmltitle{Optimal Regret Analysis of Thompson Sampling in Stochastic Multi-armed Bandit Problem with Multiple Plays}

\icmlauthor{Junpei Komiyama}{junpei@komiyama.info}
\icmlauthor{Junya Honda}{honda@stat.t.u-tokyo.ac.jp}
\icmlauthor{Hiroshi Nakagawa}{nakagawa@dl.itc.u-tokyo.ac.jp}
\icmladdress{The University of Tokyo, Japan}

\icmlkeywords{Multi-armed bandit problem, Stochastic bandit problem, Online learning}

\vskip 0.3in
]

\begin{abstract} 
We discuss a multiple-play multi-armed bandit (MAB) problem in which several arms are selected at each round.
Recently, Thompson sampling (TS), a randomized algorithm with a Bayesian spirit, has attracted much attention for its empirically excellent performance, and it is revealed to have an optimal regret bound in the standard single-play MAB problem. 
In this paper, we propose the multiple-play Thompson sampling (MP-TS) algorithm, an extension of TS to the multiple-play MAB problem, and discuss its regret analysis.
We prove that MP-TS for binary rewards has the optimal regret upper bound that matches the regret lower bound provided by  Anantharam et al.\,(1987). Therefore, MP-TS is the first computationally efficient algorithm with optimal regret. 
A set of computer simulations was also conducted, which compared MP-TS with state-of-the-art algorithms. We also propose a modification of MP-TS, which is shown to have better empirical performance.
\end{abstract} 

\section{Introduction}

The multi-armed bandit (MAB) problem is one of the most well-known instances of sequential decision-making problems in uncertain environments, which can model many real-world scenarios.
The problem involves conceptual entities called arms.
At each round, the forecaster draws one of $K$ arms and receives a corresponding reward.
The aim of the forecaster is to maximize the cumulative reward over rounds,
 and the forecaster's performance is usually measured by a regret, which is the gap between his or her cumulative reward and that of an optimal drawing policy.
Throughout the rounds, the forecaster faces an
``exploration vs.\;exploitation'' dilemma. 
On one hand, the forecaster wants to exploit the information that he or she has gathered up to the previous round by selecting  seemingly good arms.
On the other hand, there is always a possibility that the other arms have been underestimated,
which motivates him or her to explore seemingly bad arms
in order to gather their information.
To resolve this dilemma, the forecaster uses an algorithm to control the number of draws for each arm. 

In the stochastic MAB problem, which is the most widely studied version of the MAB problem, it is assumed that each arm is associated with a distinct probability distribution.
While there have been many theoretical studies on the infinite setting in which future rewards are geometrically discounted (e.g., the Gittins index \cite{git74}), recent availability of massive data has led to a finite horizon setting in which every reward has the same importance. In this work, we focus on the latter setting. 

There has been significant progress in this setting of the MAB problem.
In particular, the upper confidence bound (UCB) algorithm \cite{auerfinite}
has been widely used and studied for its computational simplicity and customizability.
Whereas the coefficient of the leading logarithmic term in UCB is larger than the theoretical lower bound given by \citet{LaiRobbins1985},
algorithms have been proposed that achieve this bound, such as DMED \cite{HondaDMED}, $\mathcal{K}_\text{inf}$, and KL-UCB \cite{klucb2}.

Moreover, Thompson sampling (TS) \cite{thompsonsampling} has recently attracted attention
for its excellent performance
\cite{scottmodern,empiricalthompson}
and it has been revealed to be applicable to even a wider class of problems
\cite{DBLP:conf/icml/AgrawalG13,DBLP:conf/nips/RussoR13,DBLP:conf/nips/OsbandRR13,DBLP:conf/aaai/KocakVMA14,DBLP:conf/colt/GuhaM14}.
Thompson sampling is 
an old heuristic that has a spirit of Bayesian inference
and selects an arm based on posterior samples of the expectation of each arm.
It has been shown that TS has an optimal regret bound \cite{DBLP:journals/jmlr/AgrawalG12,kaufmannthompson,shiprafurther}.

\subsection{Multiple-play MAB problem}
The literature mentioned above has specifically dealt with the MAB problem in which a single arm is selected and drawn at each round. Let us call this problem single-play MAB (SP-MAB). While the SP-MAB problem is indisputably important as a canonical problem, in many practical situations multiple entities corresponding to arms are selected at each round.
We call the MAB problem in which several arms can be selected multiple-play MAB (MP-MAB).
Examples of the situations that can be modeled as an MP-MAB problem include the followings.
\begin{itemize}
\vspace{-0.5em}
 \item \textbf{Example 1 (placement of online advertisements):} a web site has several slots where advertisements can be placed. Based on each user's query, there is a set of candidates of relevant advertisements from which web sites can select to display. The effectiveness of advertisements varies: some advertisements are more appealing to the user than others. With the standard model in online advertising, it is assumed that each advertisement is associated with a click-through-rate (CTR), which is the number of clicks per view. Since web sites receive revenue from clicks on advertisements, it is natural to maximize it, which can be considered as an instance of an MP-MAB problem in which advertisements and clicks correspond to arms and rewards, respectively.  
\vspace{-0.5em}
 \item \textbf{Example 2 (channel selection in cognitive radio networks \cite{DBLP:conf/infocom/HuangLD08}):} a cognitive radio is an adaptive scheme for allocating channels, such as wireless network spectrums. There are two kinds of users: primary and secondary. Unlike primary users, secondary users do not have primary access to a channel but can take advantage of the vacancies in primary access and opportunistically exploit instantaneous spectrum availability when primary users are idle. However, the availabilities of channels are not easily known. Usually, secondary users have access to multiple channels. They can enhance their communication efficiency by adaptively estimating the availability statistics of the channels, which can be considered as an MP-MAB problem in which channels and the permission of communication are arms and rewards, respectively.
\vspace{-0.5em}
\end{itemize}

There have been several studies on the MP-MAB problem. \citet{anantharam1987asymptotically} derived an asymptotic lower bound on the regret for this problem and proposed an algorithm to achieve this bound. Because their algorithm requires certain statistics that are difficult to compute, efficiently computable MP-MAB algorithms have also been extensively studied. \citet{weichencmab} extended a UCB-based algorithm to a multiple-play case with combinatorial rewards and \citet{gopalancomplex} extended TS to a wide class of problems. Although both papers provide a logarithmic regret bound, the constant factors of these regret bounds do not match the lower bound. Therefore, it is unknown whether the optimal regret bound for the MP-MAB problem is achievable by using a computationally efficient algorithm.

The main difficulty in analyzing the MP-MAB problem lies in the fact that
the regret depends on the combinatorial structure of arm draws.
More specifically,
an algorithm with the optimal bound on the number of draws of suboptimal arms
does not always ensure the optimal regret bound unlike the SP-MAB problem.

\textbf{Contribution:} Our contributions are as follows. 
\vspace{-1em}
\begin{itemize}
  \item \textbf{TS-based algorithm for the MP-MAB problem and its optimal regret bound:}
the first and main contribution of this paper is an extension of TS to the multiple play case, which we call MP-TS. We prove that MP-TS for binary rewards achieves an optimal regret bound. To the best of our knowledge, this paper is the first to provide a computationally efficient algorithm in the MP-MAB problem with the optimal regret bound by \citet{anantharam1987asymptotically}.
\vspace{-0.5em}
\item \textbf{Novel analysis technique:}
to solve the difficulty in the combinatorial structure of the MP-MAB problem,
we show that the independence of posterior samples among arms in TS is a key property
 for suppressing the number of simultaneous draws of several suboptimal arms,
 and the use of this property eventually leads to the optimal regret bound.
\vspace{-0.5em}
 \item \textbf{Experimental comparison among MP-MAB algorithms:} we compare MP-TS with other algorithms, and confirm its efficiency. We also propose an empirical improvement of MP-TS (IMP-TS) motivated by analyses on the regret structure of the MP-MAB problem. We confirm that IMP-TS improves the performance of MP-TS without increasing computational complexity.
\end{itemize}
\vspace{-1em}

\section{Problem Setup}
\label{sec:setup}

Let there be $K$ arms. 
Each arm $i \in [K] = \{1,2,\dots,K\}$ is associated with a probability distribution $\nu_i = \Bernoulli(\mu_i)$, $\mu_i \in (0,1)$.
At each round $t=1,2,\dots,T$, the forecaster selects a set of $L < K$ arms $I(t)$, then receives the rewards of the selected arms.
The reward $X_{i}(t)$ of each selected arm $i$ is i.i.d. samples from $\nu_i$.
 Let $N_i(t)$ be the number of draws of arm $i$ before round $t$ (i.e., $N_i(t) = \sum_{t'=1}^{t-1} \Ind\{i \in I(t')\}$, where $\Ind\{\EA\} = 1$ if event $\EA$ holds and $= 0$ otherwise.),
  and $\hatmu_i(t)$ be the empirical mean of the rewards of arm $i$ at the beginning of round $t$. 
The forecaster is interested in maximizing the sum of rewards over drawn arms.
For simplicity, we assume that all arms have distinct expected rewards (i.e., $\mu_i \neq \mu_j$ for any $i \neq j$).
We discuss the case in which $\mu_i = \mu_j$ for some $i$ and $j$ in Appendix \ref{subsec:sameexpectations}, which is in Supplementary Material.
Without loss of generality, we assume $\mu_1 > \mu_2 > \mu_3 > \dots > \mu_K$.
Of course, algorithms do not exploit this ordering.
We define optimal arms as top-$L$ arms (i.e., arms $[L]$), and suboptimal arms as the others (i.e., arms $\subopts$).
The regret, which is the expected loss of the forecaster, is defined as
\begin{align*}
\Regret(T)
&=\sum_{t=1}^T  \left(\sum_{i \in [L]} \mu_i - \sum_{i \in I(t)} \mu_i\right).
\end{align*}
The expectation of regret $\Expect[\Regret(T)]$ is used to measure the performance of an algorithm.

\section{Regret Bounds}

In this section we introduce the known lower bounds of the regret
for the SP-MAB and MP-MAB problems and discuss the relation between them.

\subsection{Regret bound for SP-MAB problem}
\label{sec:spmab}

The SP-MAB problem, which has been thoroughly studied in the fields of statistics and machine learning, is a special case of the MP-MAB problem with $L=1$. 
The optimal regret bound in the SP-MAB problem was given by \citet{LaiRobbins1985}.
They proved that, for any strongly consistent algorithm (i.e., algorithms with subpolynomial regret for any set of arms), there exists a lower bound
\begin{equation}
  \Expect[N_i(T+1)] \geq \left( \frac{1 - o(1)}{d(\mu_i, \mu_1)} \right) \log{T},
    \label{ineq:singlelower}
\end{equation}
where $d(p, q) = p \log{\left(p/q\right)} + (1-p) \log{\left((1-p)/(1-q)\right)}$ is the KL divergence between two Bernoulli distributions with expectation $p$ and $q$.
Note that when arm $i$ is drawn, the regret increases by $\Delta_{i,1}$
and the regret is written as
\begin{align}
\Expect[\Regret(T)] =  \sum_{i \neq 1} N_i(T+1) \Delta_{i,1},\label{relation_sp}
\end{align}
where $\Delta_{i,j} = \mu_j - \mu_i$.
Therefore, inequality \eqref{ineq:singlelower} directly leads to the regret lower bound
\begin{equation}
  \Expect[\Regret(T)] \geq \sum_{i \neq 1} \left( \frac{(1 - o(1))\Delta_{i,1}}{d(\mu_i, \mu_1)} \right) \log{T}.
    \label{lower_single}
\end{equation}

One may think that applying the techniques of the SP-MAB problem would directly yield an optimal bound for a more general MP-MAB problem. However, this is not the case. In short, the difficulty in analyzing the regret on the MP-MAB problem arises from the fact that the optimal bound on the number of suboptimal arm draws does not directly lead to the optimal regret. From this point forward, we focus on the MP-MAB problem in which $L$ is not restricted to one.

\subsection{Extension to MP-MAB problem}
\label{sec:optimalmpmab}

The regret lower bound in the MP-MAB problem, which is the generalization of
inequality \eqref{lower_single},
was provided by \citet{anantharam1987asymptotically}.
They first proved that,
 for any strongly consistent algorithm and suboptimal arm $i$, 
 the number of arm $i$ draws is lower-bounded as
\begin{equation}
  \Expect[N_i(T+1)] \geq \left( \frac{ 1 - o(1) }{d(\mu_i,\mu_L) } \right) \log{T}.
  \label{ineq:multidrawlower}
\end{equation}
\begin{figure}[t]
\begin{center}
\centerline{\includegraphics[scale=0.32]{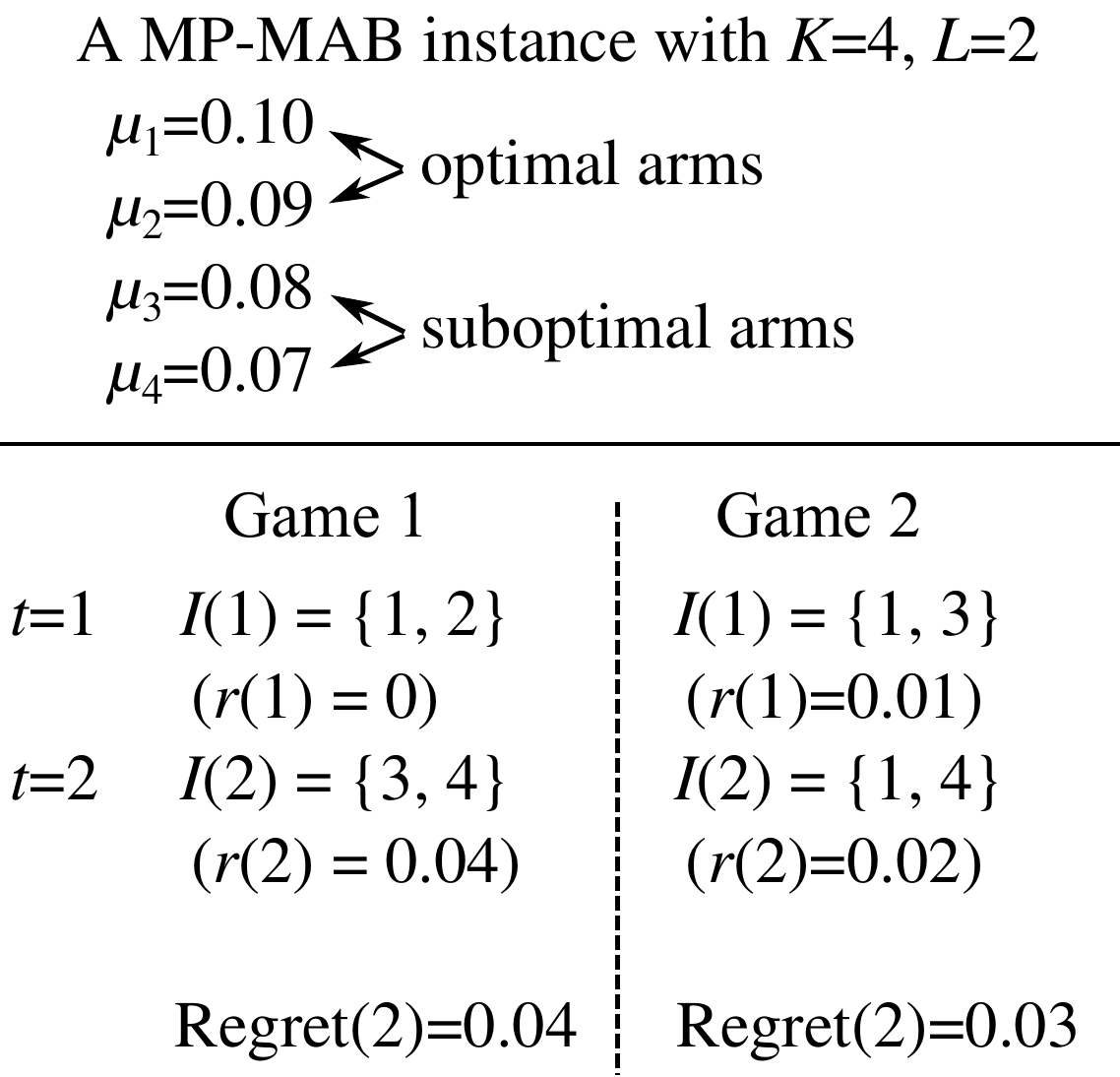}}
\caption{Two bandit games with the same set of arms. $r(t)$ is defined as the increase in the regret at round $t$. In both games 1 and 2, we have the same number of suboptimal arm draws ($N_3(2) = N_4(2) = 1$). However, the regret in games 1 and 2 are different.}
\label{fig:twocasesregret}
\end{center}
\vspace{-2em}
\end{figure}%

Unlike in the SP-MAB problem, the regret in the MP-MAB problem is not uniquely determined by the number of suboptimal arm draws. As illustrated in Figure \ref{fig:twocasesregret}, the regret is dependent on the combinatorial structure of arm draws.

Recall that a regret increase at each round is the gap of expected rewards between the optimal arms and that of the selected arms.
When a suboptimal arm is selected, one optimal arm is excluded from $I(t)$ instead of the suboptimal arm. Let the selected suboptimal arm and excluded optimal arm be $i$ and $j$, respectively. Then, we lose expected reward $\mu_j-\mu_i$.
Namely, the loss in the expected reward at each round is given by
\begin{eqnarray}
\sum_{j \in [L]} \mu_j - \sum_{i \in I(t)} \mu_i
  & = & \sum_{j \in [L] \setminus I(t)} \mu_j - \sum_{i \in I(t) \setminus [L]} \mu_i  \label{ineq:eqsmallest} \nn
  & \geq & \sum_{i \in I(t) \setminus [L]} (\mu_L - \mu_i), \label{ineq:lowersmallest}
\end{eqnarray}
where we used the fact $\mu_j \geq \mu_L$ for any optimal arm $j$.
From this relation, the regret is expressed as
\begin{align}
\Regret(T)
&\ge \sum_{t=1}^T
\sum_{i \in I(t) \setminus [L]} (\mu_L - \mu_i)\nn
&=
\sum_{i \in [K] \setminus [L]} (\mu_L - \mu_i)N_i(T+1)
\label{lower_regret}
\end{align}
which, combined with \eqref{ineq:multidrawlower},
leads to the regret lower bound by \citet{anantharam1987asymptotically}
that any strongly consistent algorithm satisfies
\begin{equation}
  \Expect[\Regret(T)] \geq \sum_{i \in \subopts} \frac{ (1 - o(1)) \Delta_{i,L} }{d(\mu_i, \mu_L) } \log{T}.
  \label{ineq:multiregretlower}
\end{equation}

\begin{algorithm}[t]
 \caption{Multiple-play Thompson sampling (MP-TS) for binary rewards}
 \label{alg:mpts}
\begin{algorithmic}
   \STATE Input: \# of arms\,$K$, \# of selection $L$
   \FOR{$i = 1,2,\dots,K$}
     \STATE $A_{i}, B_{i} = 1,1 $
   \ENDFOR 
   \STATE $t \leftarrow 1$.
   \FOR{$t = 1,2,\dots,T$}
     \FOR{$i = 1,2,\dots,K$}
       \STATE $\thetai(t) \sim \Beta(A_i, B_i)$ 
     \ENDFOR
       \STATE $\selected = $ top-$L$ arms ranked by $\thetai(t)$.
       \FOR{$i \in I(t)$}
         \IF{$X_{i}(t) = 1$}
           \STATE $A_{i} \leftarrow A_{i} + 1$
         \ELSE
           \STATE $B_{i} \leftarrow B_{i} + 1$
         \ENDIF
       \ENDFOR
   \ENDFOR
\end{algorithmic}
\end{algorithm}

\subsection{Necessary condition for an optimal algorithm}
\label{subsec:condoptimalregret}

In Sections \ref{sec:spmab} and \ref{sec:optimalmpmab},
we saw that the derivations of the regret bounds are analogous between the SP-MAB and MP-MAB problems.
However, there is a difference in
the relation between the regret and $N_i(T)$, the number of draws of suboptimal arms,
is given as equation \eqref{relation_sp} in the SP-MAB problem,
 whereas it is given as inequality \eqref{lower_regret} in the MP-MAB problem.
This means that, an algorithm achieving the asymptotic lower bound \eqref{ineq:multidrawlower}
on $N_i(T)$
does not always achieve the asymptotic regret bound \eqref{ineq:multiregretlower}.

When suboptimal arm $i$ is selected, one of the optimal arms is pushed out instead of arm $i$,
and the regret increases by the difference between the expected rewards of these two arms.
The best scenario is that, arm $L$, which is the optimal arm with the smallest expected reward,
is almost always the arm pushed out instead of a suboptimal arm.
For this scenario to occur,
it is necessary to ensure that
at most one suboptimal arm is drawn
for almost all rounds
because, if two suboptimal arms are selected,
at least one arm in $[L-1]$ is pushed out.

In the next section, we propose an extension of TS to the MP-MAB problem,
 and explain that it has a crucial property for suppressing this simultaneous draw of two suboptimal arms.

\noindent\textbf{Remark:} %
Corollary 1 of \citet{gopalancomplex} shows
the achievability of the bound in the RHS of
\eqref{ineq:multidrawlower} on the number of draws of suboptimal arms.
Whereas this does not lead to the optimal regret bound as discussed above,
they originally derived in Theorem 1 an $O(\log T)$ bound on the number of each suboptimal action
(that is, each combination of arms including suboptimal ones)
for a more general setting of MP-MAB.
Thus, we can directly use this bound to derive a better regret bound.
However, to show the optimality in the sense of regret
it is necessary to prove that
there are at most $o(\log T)$ rounds such that
an arm in $[L-1]$ is pushed out.
Therefore, it still requires further discussion to derive the
optimal regret bound of TS.
Note also that the regret bound by \citet{gopalancomplex} is restricted to the case that
the prior has a finite support and the true parameter is in the support, and thus their analysis requires some approximation scheme for dealing Bernoulli rewards.

\section{Multiple-play Thompson Sampling Algorithm}
\label{subsec:alg}

Algorithm \ref{alg:mpts} is our MP-TS algorithm.
While TS for single-play selects the top-1 arm based on a posterior sample $\theta_i(t)$,
MP-TS selects the top-$L$ arms ranked by the posterior sample $\theta_i(t)$. 
Like \citet{kaufmannthompson} and \citet{shiprafurther}, we set the uniform prior on each arm.

In Section \ref{subsec:condoptimalregret}, we discussed
that the necessary condition to achieve the optimal regret bound is to suppress the simultaneous draws of
two or more suboptimal arms, which characterizes the difficulty of the MP-MAB problem.

Note that it is easy to extend other asymptotically optimal SP-MAB algorithms, such as KL-UCB, to the MP-MAB problem.
Nevertheless, we were not able to prove the optimality of these algorithms
for the MP-MAB problem though the achievability of the bound
\eqref{ineq:multidrawlower} on $N_i(T)$ is easily proved, and the simulation results in Section \ref{sec:experiment} also imply their achievability of the regret bound.
This is because TS has quite a plausible property to suppress simultaneous draws as we discuss below.

Before the exact statement in the next section, we give an intuition for the natural extension of TS (or other asymptotically optimal SP-MAB algorithms)
can have the optimal regret in the MP-MAB problem.
Roughly speaking, a bandit algorithm with a logarithmic regret
draws a suboptimal arm with probability $O(1/t)$
at the $t$-th round, which amounts to $O(\sum_{t=1}^T 1/t)=O(\log T)$ regret.
Thus, two suboptimal arms are drawn at the same round
with probability $O(1/t^2)$, which amounts to
$O(\sum_{t=1}^T 1/t^2)=O(1)$ total simultaneous draws, provided that
each suboptimal arm is selected independently.

In TS, the score $\theta_i(t)$ for the choice of arms is generated randomly at each round
from the posterior independently between each arm,
which enables us to bound simultaneous draws as the above intuition.
On the other hand, in KL-UCB (or in other index policies),
the UCB score for the choice of arms is deterministic given the past results of rewards,
which means that the scores of suboptimal arms may behave quite similarly
in the worst case on the past rewards.

\section{Optimal Regret Bound}
\label{sec:result}

In this section, we state the main theoretical result (Theorem \ref{thm:mainregret}).
The analysis that leads to this theorem is discussed in Section \ref{sec:analysis}.

\begin{theorem} {\rm (Regret upper bound of MP-TS)}
For any sufficiently small $\epsilonOne>0, \epsilon_2>0$, the regret of MP-TS is upper-bounded as
\begin{multline*}
\Expect[\Regret(T)] \leq \sum_{i \in \subopts} \left( \frac{(1+\epsilonOne) \Delta_{i,L} \log{T}}{d(\mu_i, \mu_L)}  \right) \\ + C_a(\epsilonOne, \mu_1,\mu_2,\dots,\mu_K) + C_b(T, \epsilon_2, \mu_1,\mu_2,\dots,\mu_K), 
\end{multline*}
where, $C_a = C_a(\epsilonOne, \mu_1,\mu_2,\dots,\mu_K)$ is a constant independent on $T$
and is $O(\epsilonOne^{-2})$ when we regard $\{\mu_i\}_{i=1}^K$ as constants.
The value $C_b = C_b(T, \epsilon_2, \mu_1,\mu_2,\dots,\mu_K)$ is a function of $T$,
which, by choosing proper $\epsilon_2$, grows at a rate of $O(\log \log{T}) = o(\log{T})$.
\label{thm:mainregret}
\end{theorem}

By letting $\epsilonOne=O((\log T)^{-1/3})$
we obtain
\begin{align}
\Expect[\Regret(T)] \leq \sum_{i \in \subopts} \frac{\Delta_{i,L}\log T}{d(\mu_i, \mu_L)}
+O((\log T)^{2/3})
\label{ineq:regretmain}
\end{align}
and we see that
MP-TS achieves the asymptotic bound in \eqref{ineq:multiregretlower}.

\noindent\textbf{Expected regret and high-probability regret:}
\citet{anantharam1987asymptotically}
originally derived a regret lower bound
in a stronger form than \eqref{ineq:multiregretlower}
such that for any $\epsilon>0$, the regret of a strongly consistent algorithm is lower-bounded as
\begin{align*}
\lim_{T\to\infty}\Pr\left[\frac{\Regret(T)}{\log T}
\ge
\sum_{i \in \subopts} \frac{ (1 - \epsilon) \Delta_{i,L} }{d(\mu_i, \mu_L) }
\right]=1.
\end{align*}
Combining this with \eqref{ineq:regretmain} we can easily see that
MP-TS satisfies
\begin{align}
\lim_{T\to\infty}
\Pr\left[
\frac{\Regret(T)}{\log T} \leq \sum_{i \in \subopts} \frac{(1+\epsilon)\Delta_{i,L}}{d(\mu_i, \mu_L)}
\right]=1,\label{optimality_prob}
\end{align}
that is, MP-TS is also asymptotically optimal in the sense of
high probability.
Since an algorithm satisfying \eqref{optimality_prob}
is not always optimal in the sense of expectation,
our result, the expected optimal regret bound, is also stronger in this sense
than the high-probability bound by \citet{gopalancomplex}.

\vspace{-0.5em}
\section{Regret Analysis}
\label{sec:analysis}
\vspace{-0.5em}

We first define some additional notation that are useful for our analysis in Section \ref{subsec:addnotation} then analyze the regret bound in Section \ref{subsec:regretanalysis}. The proofs of all the lemmas, except for Lemma \ref{lem:regretfourdecomposition}, are given in the Appendix.

\vspace{-0.5em}
\subsection{Additional notation}
\label{subsec:addnotation}
\vspace{-0.5em}

Let $\mu_L^{(-)} = \mu_L - \epsilononed$ and
$\mu_i^{(+)} = \mu_i + \epsilononed$
for $\epsilononed>0$ and $i \in \subopts$.
We assume $\epsilononed$ to be sufficiently small such that $\mu_L^{(-)} \in (\mu_{L+1}, \mu_L)$ and $\mu_i^{(+)} \in (\mu_i, \mu_L)$.
We also define $N_i^{\mathrm{suf}}(T) = \frac{\log{T}}{d(\mu_i^{(+)}, \mu_L^{(-)})}$.
Intuitively, $N_i^{\mathrm{suf}}(T)$ is the sufficient number of explorations to make sure that arm $i$ is not as good as arm $L$.

\textbf{Events:}
Now, let $\tmmax{i\in S}{m}a_i$ denote the $m$-th largest element of $\{a_i\}_{i \in S} \in \mathbb{R}^{|S|}$,
that is, $\dmmax{i\in S}{m}a_i = \max_{S'\subset S:|S'|=m}\min_{i\in S'}a_i$.
We define $\theta^*(t)=\tmmax{i\in [K]}{L}\theta_i(t)$ as the $L$-th largest posterior sample at round $t$
 (i.e., the minimum posterior sample among the selected arms),
 and $\thetaijsst=\tmmax{k \in [K]\setminus \{i,j\}}{L-1}\theta_k(t)$ as the $(L-1)$-th largest posterior sample at round $t$
 except for arms $i$ and $j$.
Moreover, let $\nu = \frac{\mu_{L-1}+\mu_L}{2}$.
Let us define the following events.
\begin{eqnarray*}
  \EA_i(t) & = & \{ i \in I(t) \}, \\
  \EB(t) & = & \{ \theta^*(t) \geq \mu_L^{(-)} \}, \\ 
  \EC_i(t) & = & \bigcap_{j \in \lisubopts} \{ \thetaijsst \geq \nu \},\\ 
  \ED_i(t) & = & \{ N_i(t) < N_i^{\mathrm{suf}}(T) \}. 
\end{eqnarray*}
Event $\EA_i(t)$ states that arm $i$ is sampled at round $t$, and $\ED_i(t)$ states that arm $i$ has not been sampled sufficiently yet. The complements of $\EB(t)$ and $\EC_i(t)$ are related to the underestimation of optimal arms. Since the optimal arms are sampled sufficiently, $\EB^c(t)$ or $\EC_i^c(t)$ should not occur very frequently.

\subsection{Proof of Theorem \ref{thm:mainregret}} 
\label{subsec:regretanalysis}

We first decompose the regret to the contribution of each arm.
Recall that, the regret increase by drawing suboptimal arm $i$ is determined by the optimal arm excluded in the selection set $I(t)$.
 Formally, for suboptimal arm $i$, let
\begin{equation}
 \Delta_i(t) =\begin{cases}
    (\max_{j\in [L] \setminus I(t)} \mu_j) - \mu_i & \mathrm{if} \hspace{0.5em} I(t)\neq [L], \\
    0 &\mathrm{otherwise},
  \end{cases}\label{def_delta}
\end{equation}
and 
\begin{equation*}
 \Regret_i(T) =  \sum_{t=1}^T \Ind\{i \in I(t)\} \Delta_i(t).
\end{equation*}
From inequality (\ref{ineq:eqsmallest}) the following inequality is easily derived
\begin{equation*}
 \Regret(T) \leq \sum_{i \in \subopts}\Regret_i(T).
\end{equation*}
We next decompose $\Regret_i(T)$ into several terms by using events $\EA$--$\ED$.
After giving bounds for these terms, we finally give the total regret bound, which proves Theorem \ref{thm:mainregret}.
Note that, in bounding the deviation of Bernoulli means and Beta posteriors in the Appendix, our analysis borrowed some techniques developed in the context of the SP-MAB problem, mostly from \citet{shiprafurther}, and some from \citet{DBLP:conf/aistats/HondaT14}. 

\begin{lemma}
The regret by drawing suboptimal arm $i>L$ is decomposed as:
\label{lem:regretfourdecomposition}
\begin{multline*}
 \Regret_i(T) \leq \underbrace{ \sum_{t=1}^T \Ind\{\EB^c(t)\} }_{\mathrm{(A)}} + \underbrace{ \sum_{t=1}^T \Ind\{\EA_i(t), \EC_i^c(t)\} }_{\mathrm{(B)}} \\ + \underbrace{ \sum_{j \in \lisubopts} \sum_{t=1}^T \Ind\{\EA_i(t), \EC_i(t), \ED_i(t), \EA_j(t)\} }_{\mathrm{(C)}} \\ + \underbrace{ \sum_{t=1}^T \Ind\{\EA_i(t), \EB(t), \ED_i^c(t)\} }_{\mathrm{(D)}} + N_i^{\mathrm{suf}}(T) \Delta_{i,L},
\end{multline*}
where, for example, $\{\EA, \EB\}$ abbreviates $\{\EA \cap \EB\}$.
\end{lemma}%
Roughly speaking, 
\begin{itemize}
\vspace{-0.9em}
 \item Term (A) corresponds to the case in which, some of the optimal arms are under-estimated.  
\vspace{-0.9em}
 \item Term (B) corresponds to the case in which, arm $i$ is selected and some of the arms in $[L-1]$ are under-estimated.   
\vspace{-0.9em}
 \item Term (C) corresponds to the case in which, arm $i \in \subopts$ and  $j \in \lisubopts$ are simultaneously drawn. In particular, term (C) is unique in the MP-MAB problem that causes additional regret increase,
and in analyzing this term we fully use the fact that the samples of the posterior distributions on the arms are independent of each other.
 \item Term (D) corresponds to the case in which, arm $i$ is selected after it is sufficiently explored. 
\vspace{-1.0em}
\end{itemize}
\begin{proof}[Proof of Lemma \ref{lem:regretfourdecomposition}]
The contribution of suboptimal arm $i$ to the regret is decomposed as follows.
By using the fact $\Delta_i(t) \leq 1$ and the following decomposition of an event
\begin{align*}
\lefteqn{
\EA_i(t) \subset \EB^c(t) \cup \{\EA_i(t), \EC_i^c(t)\} \cup \{\EA_i(t), \EB(t), \EC_i(t)\}
} \\
 & \subset \EB^c(t) \cup \{\EA_i(t), \EC_i^c(t)\}   \\
 & \hspace{5em} \cup \{\EA_i(t), \EB(t), \ED_i^c(t)\} \cup \{\EA_i(t), \EC_i(t), \ED_i(t)\},
\end{align*}
we have
\begin{align}
 \Regret_i(T) = \sum_{t=1}^T \Ind\{\EA_i(t)\} \Delta_i(t) \hspace{-13em} \nn
  & \leq \sum_{t=1}^T \Ind\{\EB^c(t)\} + \sum_{t=1}^T \Ind\{ \EA_i(t), \EC_i^c(t) \} \nn
  &  \hspace{2em} + \sum_{t=1}^T \Ind\{\EA_i(t), \EB(t), \ED_i^c(t)\} \nn
  &  \hspace{2em} + \sum_{t=1}^T \Ind\{\EA_i(t), \EC_i(t), \ED_i(t)\} \Delta_i(t). \label{ineq:regretterms1}
\end{align}
Recall that $\Delta_i(t)$ is defined as \eqref{def_delta}. At each round, when $L$ and all suboptimal arms, except for $i$, are not selected, then $I(t)= \{1,2,\dots,L-1,i\}$; $\Delta_i(t) = \Delta_{i,L} $.
Therefore,
\begin{align}
\lefteqn{
  \sum_{t=1}^T \Ind\{\EA_i(t), \EC_i(t), \ED_i(t)\} \Delta_i(t) 
} \nn
  & \leq \sum_{t=1}^T \Ind\{\EA_i(t), \EC_i(t), \ED_i(t)\} \Delta_{i,L} \nn
  & \hspace{1em} + \sum_{t=1}^T \Ind\{\EA_i(t), \EC_i(t), \ED_i(t), \bigcup_{j \in \lisubopts} \EA_j(t)\} \nn
 &  \leq \sum_{t=1}^T \Ind\{\EA_i(t), \ED_i(t)\} \Delta_{i,L} \nn
& \hspace{1em}+ \sum_{j \in \lisubopts} \sum_{t=1}^T  \Ind\{\EA_i(t), \EC_i(t), \ED_i(t), \EA_j(t)\} \nn
 &  \leq N_i^{\mathrm{suf}}(T) \Delta_{i,L} 
\nn & \hspace{1em} + \sum_{j \in  \lisubopts}  \sum_{t=1}^T  \Ind\{\EA_i(t), \EC_i(t), \ED_i(t), \EA_j(t)\}.  \label{ineq:regretterms3}
\end{align}
Summarizing \eqref{ineq:regretterms1} and \eqref{ineq:regretterms3} completes the proof.
\end{proof}

The following lemma bounds terms (A)--(D).
\begin{lemma} {\rm (Bounds on individual terms)}
Let $\epsilon_2 > 0$ be arbitrary.
For sufficiently small $\epsilononed$ and $\epsilon_2$, the four terms are bounded in expectation as:
\begin{eqnarray}
 \Expect[\mathrm{(A)}] \hspace{-0.5em} & = & \hspace{-0.5em} O\left(\frac{1}{(\mu_L-\mu_L^{(-)})^2}\right) = O\left(\frac{1}{\epsilononed^2}\right), \label{ineq:boundterma}\\
 \Expect[\mathrm{(B)}] \hspace{-0.5em} & = & \hspace{-0.5em} O(\log{\log{T}}), \label{ineq:boundtermb} \\ 
 \Expect[\mathrm{(C)}] \hspace{-0.5em} & \leq & \hspace{-1.4em} \sum_{j \in \lisubopts} \hspace{-1em} \frac{\left(\epsilon_2 +  4 T ^ {- \frac{\epsilon_2 \Delta_{L,L-1}^2}{8} } \hspace{-0.2em}\right) \log{T}}{d(\mu_i, \mu_L)}  \hspace{-0.1em} + \hspace{-0.1em} O(1), \hspace{-1em} \nonumber\\ \text{and} \label{ineq:boundtermc} \\
 \Expect[\mathrm{(D)}] \hspace{-0.7em} & \leq & \hspace{-0.7em} 2 \hspace{-0.2em} + \hspace{-0.2em} \frac{1}{d(\mu_i^{(+)}, \mu_i)} = O\left(\frac{1}{\epsilononed^2}\right). \label{ineq:boundtermd} 
\end{eqnarray}
\label{lem:fourterms}
\end{lemma}%
\vspace{-1em}
The proof of Lemma \ref{lem:fourterms} is in Appendix \ref{subsec:forterms}.
Lemma \ref{lem:fourterms} states that terms (A), (B), and (D) are $O(1/\epsilononed^2)$.
Moreover, the following lemma bounds term (C).
\begin{lemma} {\rm (Asymptotic convergence of $\epsilon_2$-dependent factor)}
By choosing an $O((\log\log T)/\log T)$ value of $\epsilon_2$,
 we obtain $\Expect[\mathrm{(C)}] = O(\log{\log{T}})$.
\label{lem:termcconvergence}
\end{lemma}%
The proof of Lemma \ref{lem:termcconvergence} is in Appendix \ref{subsec:proofe2lim}.
Now it suffices to evaluate
$N_i^{\mathrm{suf}}(T) = \frac{\log{T}}{d(\mu_i^{(+)}, \mu_L^{(-)})}$
to complete the proof.
From the convexity of KL divergence
there exists a constant $c_i = c_i(\mu_i, \mu_L)>0$ such that
\begin{align*}
d(\mu_i^{(+)}, \mu_L^{(-)})
=
d(\mu_i+\epsilononed, \mu_L-\epsilononed)
\ge (1-c_i\epsilononed)d(\mu_i, \mu_L)
\end{align*}
and therefore
\begin{align*}
\lefteqn{
\Expect[\Regret(T)] \hspace{-0.3em} \leq \hspace{-1em} \sum_{i \in \subopts} \hspace{-1em} \Expect[\Regret_i(T)] \leq \hspace{-1em} \sum_{i \in \subopts} \hspace{-1em} \Expect\left[ \sum_{t=1}^T \Ind\{\EA_i(t)\} \Delta_i(t)  \right]
} \nn
 & \leq  \sum_{i \in \subopts} \hspace{-1em} \left\{ \Expect\left[\mathrm{(A)} + \mathrm{(B)} + \mathrm{(C)} + \mathrm{(D)}\right] + N_i^{\mathrm{suf}}(T)\Delta_{i,L} \right\} \nn
 & \leq \underbrace{ \hspace{-0.7em}  \sum_{i \in \subopts} \frac{\Delta_{i,L} \log{T}}{(1-c_i \epsilononed)  d(\mu_i, \mu_L)} }_{\text{main term}}  + \underbrace{ O\left(\frac{1}{\epsilononed^2}\right) }_{C_a} + \underbrace{ O(\log{\log{T}})  }_{C_b}. 
\end{align*}
Since $(1-c_i\epsilononed)^{-1}\le 1+2c_i\epsilononed$ for $c_i\epsilononed\le 1/2$,
we complete the proof of Theorem \ref{thm:mainregret} by letting
$\epsilonOne<1/2$ and $\epsilononed=\epsilonOne/\max_{i\in\subopts}c_i=\Theta(\epsilonOne)$.
\qed

\begin{figure*}[t!]
\vspace{-0.5em}
\begin{center}
  \setlength{\subfigwidth}{.33\linewidth}
  \addtolength{\subfigwidth}{-.33\subfigcolsep}
  \begin{minipage}[t]{\subfigwidth}
  \centering
 \subfigure[Scenario 1]{\includegraphics[scale=0.6]{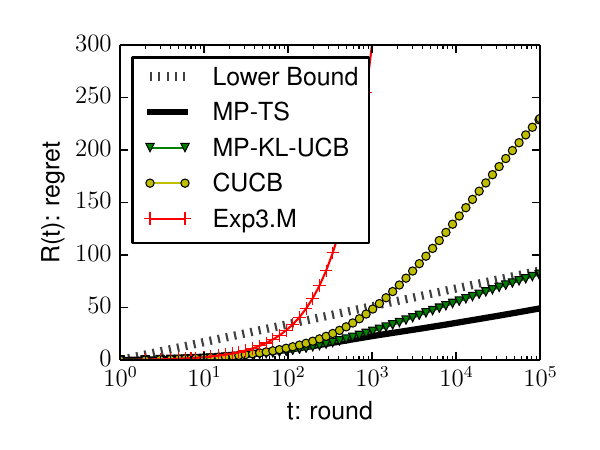}}
 \end{minipage}\hfill
  \begin{minipage}[t]{\subfigwidth}
  \centering
 \subfigure[Scenario 2]{\includegraphics[scale=0.6]{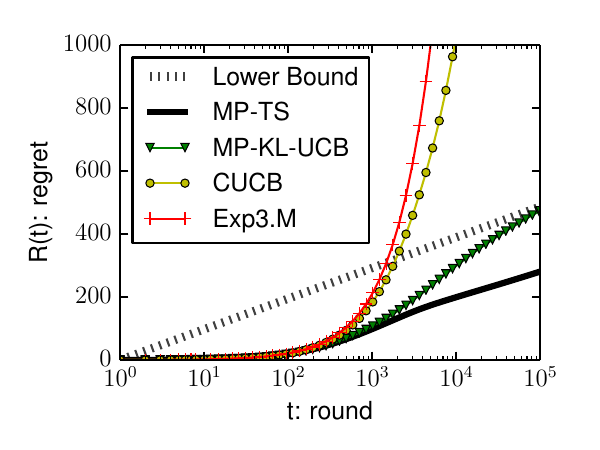}}
  \end{minipage}\hfill
  \begin{minipage}[t]{\subfigwidth}
  \centering
 \subfigure[Scenario 3]{\includegraphics[scale=0.6]{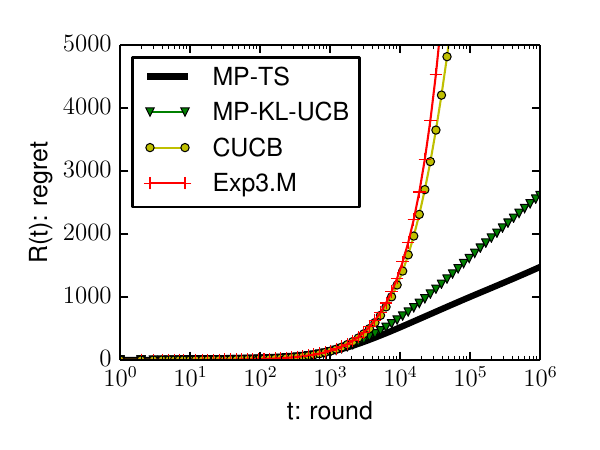}}
  \end{minipage}
\end{center}
\vspace{-1em}
  \caption{Regret-round plots of algorithms. The regret in Scenarios 1 and 2 are averaged over $10,000$ runs, and the regret in Scenario 3 is averaged over $1,000$ runs. ``Lower Bound'' is the leading $\Omega(\log{T})$ term of the RHS of inequality \eqref{ineq:multiregretlower}. We do not show Lower Bound in Scenario 3 because the coefficient of the bound can sometimes be quite large (i.e., in some runs, $1/d(\mu_{L+1}, \mu_L)$ is large).}
 \label{fig:regret}
\vspace{-1em}
\end{figure*}%

\vspace{-1em}
\section{Experiment}
\label{sec:experiment}
\vspace{-0.5em}

We ran a series of computer simulations\footnote{The source code of the simulations is available at https://github.com/jkomiyama/multiplaybanditlib.} to clarify the empirical properties MP-TS. The simulations involved the following three scenarios. In Scenarios 1 and 2, we used fixed arms similar to that of \citet{GarivierKLUCB}, and  Scenario 3 is based on a click log dataset of advertisements on a commercial search engine.

\vspace{-0.5em}
\textbf{Algorithms:}
the simulations involved MP-TS, Exp3.M \cite{DBLP:conf/alt/UchiyaNK10}, CUCB \cite{weichencmab}, and MP-KL-UCB. Exp3.M is a state-of-the-art adversarial bandit algorithm for the MP-MAB problem\footnote{Note that, Exp3.M is designed for the adversarial setting in which the rewards of arms are not necessarily stationary.}.
The learning rate $\gamma$ of Exp3.M is set in accordance with Corollary 1 of \citet{DBLP:conf/alt/UchiyaNK10}.
 Note that the CUCB algorithm in the MP-MAB problem at each round draws the top-$L$ arms of the UCB indices $\hatmu_i + \sqrt{(3 \log{t})/(2 N_i(t))}$. MP-KL-UCB is the algorithm that selects the top-$L$ arms in accordance with the KL-UCB index $\sup_{q \in [\hatmu_i(t), 1]} \left\{q | N_{i}(t) d(\hatmu_i(t), q) \leq \log{t}  \right\}$.
\vspace{-0.5em}

\textbf{Scenario 1 (5-armed bandits): }
the simulations include 5 Bernoulli arms with $\{\mu_1,\dots,\mu_5\} = \{0.7,0.6,0.5,0.4,0.3\}$, and $L=2$.
\vspace{-0.5em}

\textbf{Scenario 2 (20-armed bandits): }
the simulations include 20 Bernoulli arms with $\mu_1=0.15$, $\mu_2=0.12$, $\mu_3=0.10$, $\mu_i=0.05$  for $i \in \{4,5,\dots,12\}$, $\mu_i=0.03$ for $i \in \{13,14,\dots,20\}$, and $L=3$. 
\vspace{-0.5em}

\textbf{Scenario 3 (many-armed bandits, online advertisement based CTRs): }
we conducted another set of experiments with arms whose expectations were based on the dataset provided for KDD Cup\footnote{https://www.kddcup2012.org/} 2012 track 2. The dataset involves a click log on soso.com (a large-scale search engine serviced by Tencent), which is composed of 149 million impressions (view of advertisements). We processed the data as follows. First, we excluded users of abnormally high click probability (i.e., users who had more than $1,000$ impressions and more than $0.1$ click probability) from the log. We also excluded minor advertisements (ads) that had less than $5,000$ impressions.
There are a wide variety of ads on a search engine (e.g., "rental cars", "music", etc.) and randomly picking ads from a search engine should yield a set of irrelevant ads. To address this issue, we selected popular queries that had more than $10^4$ impressions and  more than $50$ ads that appeared on the query. As a result, $80$ queries were obtained. 
The number of ads associated with each query ranged from $50$ to $105$, and the average click-through-rate (CTR, the probability that the ad is clicked) of an ad on each query ranged from 1.15\% to 6.86\%. After that, each ad was converted into a Bernoulli arm with its expectations corresponding to the CTR of the ad. 
At the beginning of each run, one of the queries was randomly selected, and the bandit simulation with the arms corresponding to the query and $L=3$ is then conducted.
This scenario was more difficult than the first two scenarios in the sense that 1) a larger number of arms were involved and 2) the reward gap among arms was very small.

The simulation results are shown in Figure \ref{fig:regret}.
In all scenarios, the tendency is the same: our proposed MP-TS performs significantly better than the other algorithms. 
 MP-KL-UCB is not as good as MP-TS, but clearly better than CUCB and Exp3.M.
While it is unclear whether the slope of the regret of MP-KL-UCB converges to the asymptotic bound or not, the slope of the regret of TS quickly approaches the asymptotic lower bound.

\subsection{Improvement of MP-TS based on the empirical means}

\begin{figure}[h]
\vspace{-1.0em}
 \begin{center}
 \includegraphics[scale=0.55]{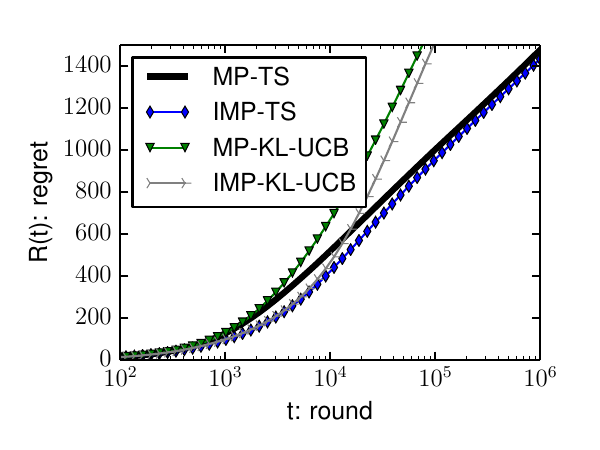}
 \end{center}
\vspace{-1.6em}
 \caption{Before/after comparison of MP-TS. All settings (except for algorithms) are the same as that of Scenario 3.}
 \label{fig:regret_imp}
\vspace{-0.6em}
\end{figure}%
We now introduce an improved version of MP-TS (IMP-TS).
In the theoretical analysis of the MP-MAB problem,
we observed that an extra loss arises
when multiple suboptimal arms are drawn at the same round.
Based on this observation,
the new algorithm selects $L-1$ arms on the basis of empirical averages and selects the last arm on the basis of TS to avoid simultaneous draws of suboptimal arms.
In other words,
this algorithm is further aimed to minimize the regret by purely exploiting the knowledge in the top-$(L-1)$ arms; thus, limiting the exploration to only one arm.
One might fear that this increase in exploitation could devastate
the balance between exploration and exploitation.
Although we provide no regret bound for the improved version of the algorithm,
we expect that this algorithm will also achieve the asymptotic bound for the following reason.
When we restrict the exploration to one arm,
the number of opportunities for an arm to be explored may decrease, say, from $T$ to $T/L$.
Still, $T/L$ opportunities are sufficient since $O(\log (T/L))=O(\log T)$.
In fact, the algorithm proposed by \citet{anantharam1987asymptotically} achieves
the asymptotic bound even though $L-1$ arms are selected based on
empirical means as in IMP-TS.
Similarly, we define an improved version of MP-KL-UCB (IMP-KL-UCB) for selecting the first $L-1$ arms on the basis of empirical averages. 
The before/after analysis of this improvement is shown in Figure \ref{fig:regret_imp}.
 One sees that, (i) MP-TS still performs better than IMP-KL-UCB, and (ii) IMP-TS reduces the regret throughout the rounds. In particular, when the number of the rounds is small ($T \sim 10^3$--$10^4$), the advantage of IMP-TS is large.

\vspace{-1em}
\section{Discussion}
\vspace{-0.6em}

We extended TS to the multiple-play setting and proved its optimality in terms of the regret. 
We considered the case in which the total reward is linear to the individual rewards of selected arms.
The analysis in this paper fully uses the independent property of posterior samples and paves the way to obtain a tight analysis on the multiple-play regret that depends on the combinatorial structure of arm selection.
We now point out two promising directions for future work.
\begin{itemize}
\vspace{-0.5em}
\item \textbf{Position-dependent factors for online advertising:} it is well-known that the CTR of an ad is dependent on its position. Taking the position-dependent factor into consideration changes the MP-MAB problem from the $L$-set selection problem to the $L$-sequence selection problem in which the position of $L$ arms matters.
For the starting point, we consider an extension of MP-TS for the cascade model \cite{DBLP:conf/wine/KempeM08,DBLP:conf/wine/AggarwalFMP08} that corrects position-dependent bias in Appendix \ref{sec:onlinead}.
\vspace{-0.5em}
\item \textbf{Non-Bernoulli distributions for general problems:} for the ease of argument, we exclusively consider the binary rewards.
The analysis by \citet{DBLP:conf/nips/KordaKM13} is useful in extending our result to the case of the 1-d exponential families of rewards. 
Moreover, extending our result to multi-parameter reward distributions \cite{BurKat96,DBLP:conf/aistats/HondaT14} is interesting.
\vspace{-0.5em}
\end{itemize}

\clearpage

\section*{Acknowledgements}

We gratefully acknowledge the insightful advice from Issei Sato and Tao Qin. We thank Yingce Xia for discussion on the evaluation of \eqref{ineq:boundtermb}. We thank Bertrand Chapleau for discussion on Bias-corrected MP-TS. We thank Zhibing Zhao for pointing out our notation error in the arXiv version. We thank the anonymous reviewers in ICML2015 for their useful comments. This work was supported in part by JSPS KAKENHI Grant Number 26106506.

\bibliographystyle{icml2015}
\bibliography{manual}
\clearpage

\appendix
\section{Appendix}

\subsection{Cases of several arms having the same expectation}
\label{subsec:sameexpectations}

Up to now, we have assumed that all arms have distinct expectations. Here, we consider cases in which some arms have the same expectations. Without loss of generality, we assume $\mu_1 \geq \mu_2 \geq,\dots,\geq \mu_K$. Let us call arms with a larger expectation than $\mu_L$ ``strictly optimal'' arms, arms with the same expectation as $\mu_L$ ``marginal'' arms, and arms with a smaller expectation than $\mu_L$ ``strictly suboptimal'' arms. Each arm is either strictly optimal, marginal, or strictly suboptimal.

\noindent
\textbf{Case 1:} Assume that all strictly optimal arms are distinct, that there is only one marginal arm, and that there are several strictly suboptimal arms with the same expectation. In this case, the regret bound of Theorem \ref{thm:mainregret} holds because our analysis deals with each suboptimal arm separately.

\noindent
\textbf{Case 2:} Assume that there is only one marginal arm, that all strictly suboptimal arms are distinct, and that there are several strictly optimal arms with the same expectation. The regret bound also holds in this case since there is a gap between each strictly suboptimal arm and each strictly optimal arm.

\noindent
\textbf{Case 3:} Assume that all strictly optimal arms and strictly suboptimal arms are distinct and that there are several marginal arms with the same expectation. Unfortunately, we were unable to perform a meaningful analysis in this case. Intuitively, as stated by Agrawal and Goyal \cite{DBLP:journals/jmlr/AgrawalG12} for SP-MAB, adding an additional marginal arm appears to require some extra exploration, which slightly increases the regret. However, the regret structure is more complex than the SP-MAB because several marginal arms can be drawn simultaneously.

In summary, our Theorem \ref{thm:mainregret} holds when the marginal arm is distinct. That is, $\mu_1 \geq \mu_2 \geq \dots \geq \mu_{L-1} > \mu_L > \mu_{L+1} \geq \dots \geq \mu_{K}$.
 
\subsection{Cascade model and position-dependent MP-MAB problem}

\begin{algorithm}[t]
 \caption{Bias-Corrected Multiple-play Thompson sampling (BC-MP-TS) for binary rewards}
 \label{alg:bcmpts}
\begin{algorithmic}
   \STATE Input: \# of arms\,$K$, \# of positions\,$L$, discount factors\,$\{\gamma_l(i)\}$
   \FOR{$i = 1,2,\dots,K$}
     \STATE $A_{i}, N_{i} = 1, 2$
   \ENDFOR 
   \STATE $t \leftarrow 1$.
   \FOR{$t = 1,2\dots,T$}
     \FOR{$i = 1,2,\dots,K$}
       \STATE $B_i \leftarrow  \max{(N_{i} - A_{i}, 1)} $
       \STATE $\thetai(t) \sim \Beta(A_i, B_i)$ 
     \ENDFOR
       \STATE Select $I_l(t)$ $(l=1,\dots,L)$ in accordance with Section \ref{subsec:optimalplacement}.
       \FOR{$l \in 1,2,\dots,L$}
         \IF{$X_{i}(t) = 1$}
           \STATE $A_{i} \leftarrow A_{i} + 1$
         \ENDIF
         \STATE $N_{i} \leftarrow N_{i} + \prod_{l'=2}^{l} \gamma_{l'} (I_{l'-1}(t))$
       \ENDFOR
   \ENDFOR
\end{algorithmic}
\end{algorithm}

\label{sec:onlinead}
In the main paper, we assumed that the rewards of arms
 are independently and identically drawn from individual distributions.
In this section, we relax this assumption and consider a wider class of the MP-MAB problem.
Remember that, one of our primary applications is multiple advertisement placement in the online advertising problem (c.f., Example 1). In this section, we interchangeably use the terms an advertisement (ad) and an arm.
It is known that the CTR of an ad depends on the environment where the ad is placed, especially on the position of the ad.
Among several models that explain this dependency on the position, the model that explains human behavior and agrees well with real data \cite{DBLP:conf/wsdm/CraswellZTR08} is the \textit{cascade} model \cite{DBLP:conf/wine/KempeM08,DBLP:conf/wine/AggarwalFMP08}, with which it is assumed that the user scans the ads from top to bottom.
Following \citet{DBLP:conf/aamas/GattiLT12}, we define the discount factor $\gamma_l(i)$ for $l\geq2$ as the probability that a user observing ad $i$ in position $l-1$ will observe the ad in the next position.
Namely, the MP-MAB problem with a discount factor is defined as a MP-MAB problem in which the arm at position $l$ yields reward $1$ with probability $\left( \prod_{l'=2}^{l} \gamma_{l'} (I_{l'-1}(t)) \right) \mu_{I_l(t)}$, where $I_l(t)$ be the arm placed at the $l$-th position at round $t$.
 Note that, when we set $\gamma_l(i) = 1$ for any position $l \in [L]$ and ad $i$, this model is reduced to the model we considered in the main paper.
In the MP-MAB problem in the main paper, the order of the $L$ arms does not matter. Whereas, under a position-dependent discount factor smaller than $1$, the order of $L$ arms matters: the problem is not the selection of an $L$-set of arms, but an $L$-sequence of arms.

\subsubsection{Thompson sampling for cascade model}

 In the cascade model, there is some probability that the arm at position $l>1$ is not drawn.
The probability that the arm at position $l$ is drawn, $\prod_{l'=2}^{l} \gamma_{l'} (I_{l'-1}(t))$, can be considered as the \textit{effective number of the draws} at position $i$.
MP-TS (Algorithm \ref{alg:mpts}) keeps $A_i$ and $B_i$, which respectively correspond to the number of rewards $1$ and $0$. The number of draws on the arm $i$ is $N_i = A_i + B_i$.
When we consider the cascade model, we need to take the effective number of draw into consideration. 
We introduce Bias-corrected MP-TS (BC-MP-TS, Algorithm \ref{alg:bcmpts}).
The crux of BC-MP-TS is that, for each arm that is selected, $N_i$ should be increased not by $1$, but by the effective number of draw for each position.
Note that, when $\gamma_l(i) = 1$, BC-MP-TS is essentially the same as MP-TS.

\subsubsection{Optimal arm selection and the regret}
\label{subsec:optimalplacement}

In general discount factor $\gamma_l(i)$, even if we have perfect information over the expectation of all arms $\{\mu_i\}_{i=1}^K$, the computation of the optimal sequence of $L$-arms at each round $t$ (optimal arm selection) appears to be computationally intractable when $K$ is large because we need to search all the possible allocation of $K$ ads over $L$ positions. In the case where $\gamma_l(i) = \gamma(i)$, \citet{DBLP:conf/wine/KempeM08} proposed a polynomial-time approximation of the optimal arm selection.
We can obtain the arm selection strategy for BC-MP-TS by using this approximation algorithm as an oracle and plugging $\{\thetai(t)\}_{i=1}^L$ as estimated expected rewards.

\textbf{Ad-independent discount factor:}
when the discount factor is independent of the ad at that position (i.e., $\gamma_l(i) = \gamma_l$), the optimal arm selection is easy: just select $\mu_l$ (i.e., $l$-th best arm) on the $l$-th position. 
We define the arm selection strategy of BC-MP-TS as placing the arm of the $l$-th largest $\theta_i$ (i.e., $I_l(t) = \tmmax{i \in [K]}{l} \thetai$) on the $l$-th position.

\textbf{Regret:}
naturally, the regret per round is defined as the difference between the expected reward of the optimal arm selection and that of an algorithm. Namely,
\begin{multline*}
\Regret(T)
 = \sum_{t=1}^T  \sum_{l=1}^L \Biggl( \prod_{l'=2}^{l} \gamma_{l'} (I_\text{opt}(l'-1)) \mu_{I_\text{opt}(l)}  \\ - \underbrace{\prod_{l'=2}^{l} \gamma_{l'} (I_{l'-1}(t))}_{\text{effective number of draw at position $l$}} \times \mu_{I_l(t)} \Biggr),
\label{def_regret_cascade}
\end{multline*}
where $(I_\text{opt}(1),\dots,I_\text{opt}(L))$ is the optimal arm selection.
In the case of the ad-independent discount factor, we conjecture that the regret lower bound should be identical to the case of no-discount factor that we analysed in the main paper (i.e., inequality \eqref{ineq:multiregretlower}). Although we do not prove any regret bound for this cascade model, the conjecture is supported by the fact that (i) by identifying the top-$L$ arm we immediately obtain the optimal arm selection, (ii) algorithms should require $\log{T}/d(\mu_i, \mu_L)$ number of effective draws to convince that suboptimal arm $i>L$ is not as good as arm $L$, and (iii) the best situation is that the simultaneous draw of several optimal arms rarely occurs: arm $L$ is pushed out instead of arm $i$, and the regret increase per an effective draw is $\mu_L - \mu_i$.
In the case of the general discount factor, the problem is subtler because a slight difference in $\{\mu_i\}$ can change the optimal arm selection.

\subsubsection{Experiment of cascade model}

\begin{figure}[t]
 \begin{center}
 \includegraphics[scale=0.6]{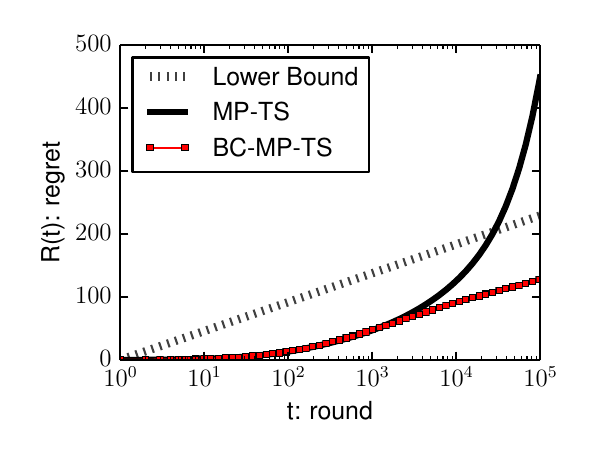}
 \end{center}
\vspace{-1.5em}
 \caption{Simulation with a discount factor. Lower Bound is the leading $\Omega(\log{T})$ term of the RHS of inequality \eqref{ineq:multiregretlower}, which we have conjectured to be the lower bound for the cascade model with the ad-independent discount factor in Section \ref{subsec:optimalplacement}. The regret is averaged over $10,000$ runs.}
 \label{fig:regret_cascade_4}
\vspace{-1em}
\end{figure}%
This simulation adapts the cascade model and involves a constant discount factor $\gamma_l(i) = 0.7$ for any position and arm. There are $9$ Bernoulli arms with $\mu_1=0.24,\mu_2=0.21,\dots,\mu_9=0.00$ and $L=3$.
In this case the optimal arm selection strategy is to choose $\{I_1(t), I_2(t), I_3(t)\}=\{\mu_1, \mu_2, \mu_3\}$ (c.f., Section \ref{subsec:optimalplacement}).
The regret of the algorithms is shown in \ref{fig:regret_cascade_4}. On one hand, MP-TS failed to have a small regret due to its ignorance to the discount factors. On the other hand, the slope of BC-MP-TS quickly approaches the conjectured Lower Bound, which is  empirical evidence of the ability of BC-MP-TS to correct the position-dependent bias. 

\subsection{Key fact and lemmas}

\begin{fact} {\rm (Chernoff bound for binary random variables)}

Let $X_1,\dots,X_n$ be i.i.d.\,binary random variables.
Let $\hat{X} = \frac{1}{n}\sum_{i=1}^n X_i$ and $\mu = \Expect[X_i]$.
Then, for any $\epsilon \in (0, 1-\mu)$,
\begin{equation*}
  \Prob( \hat{X} \geq \mu + \epsilon ) \leq \exp{\left( - d(\mu+\epsilon, \mu) n \right)}.
\end{equation*}
and, for any $\epsilon \in (0, \mu)$,
\begin{equation*}
  \Prob( \hat{X} \leq \mu - \epsilon ) \leq \exp{\left( - d(\mu-\epsilon, \mu) n \right)}.
\end{equation*}
\label{fact:chernoff}
\end{fact}

\begin{fact} {\rm (Beta-Binomial equality)}
Let $F_{\alpha,\beta}^{\mathrm{beta}}(y)$ be the cdf of the beta distribution with integer parameters $\alpha$ and $\beta$. Let $F_{n,p}^{\mathrm{B}}(\cdot)$ be the cdf of the binomial distribution with parameters $n$, $p$. Then,
\begin{equation*}
  F_{\alpha,\beta}^{\mathrm{beta}}(y) = 1 - F_{\alpha+\beta-1,y}^{\mathrm{B}}(\alpha-1), 
\end{equation*}
\end{fact}

\begin{fact} {\rm (Pinsker's inequality for binary random variables)}
For $p,q \in (0,1)$, the KL divergence between two Bernoulli distributions is bounded as: 
\begin{equation*}
d(p,q) \geq 2 (p-q)^2. 
\end{equation*}
\end{fact}

\begin{lemma} {\rm (Lemma 2 in \citet{shiprafurther})}
Let $k \in [K]$, $n \geq 0$ and $x < \mu_k$.
Let $\hatmu_{k,n}$ be the empirical average of $n$ samples from $\Bernoulli(\mu_k)$. 
Let $p_{k, n}(x) = 1 - F_{\hatmu_{k,n}n+1,(1-\hatmu_{k,n})n+1}^{\mathrm{beta}}(y)$ be the probability that the posterior sample from the Beta distribution with its parameter $\hatmu_{k,n}n+1, (1-\hatmu_{k,n})n+1$ exceeds $x$.  
Then, its average over runs is bounded as:
\begin{multline*}
  \Expect\left[\frac{1}{p_{k,n}(x)}\right] \leq \\ \begin{cases}
    1 + \frac{3}{\Delta_k(x)} & \hspace{-3em} (\text{$n < 8/\Delta_k(x)$}) \\
    1 + \Theta \Biggl(e^{-\Delta_k(x)^2 n / 2} + \frac{1}{(n+1)\Delta_k(x)^2} e^{-D_k(x) n} \\ \hspace{5em} + \frac{1}{e^{\Delta_k(x)^2 n / 4}-1} \Biggr) & \hspace{-3em} (\text{$n \geq 8/\Delta_k(x)$}),
  \end{cases}
\end{multline*}
where $\Delta_k(x) = \mu_k - x, D_k(x)=d(x, \mu_k)$.
\label{lem:bernoullideviatebound}
\end{lemma}

In the proof of Lemma \ref{lem:fourterms} 
we use the following Lemmas \ref{lem_theta_gen}, \ref{lem_avgdeviation}, and \ref{lem_chernoffdeviation} several times.
Lemma \ref{lem_theta_gen} is essentially the combination of the existing techniques of \citet{shiprafurther} and \citet{DBLP:conf/aistats/HondaT14}. Lemmas \ref{lem_avgdeviation} and \ref{lem_chernoffdeviation} are also existing techniques that appear in several previous analyses in Bayesian bandits with Bernoulli arms. 

\begin{lemma} \label{lem_theta_gen}
Let $k \in [K]$, $z<\mu_k$ be arbitrary, $\ES(t)$, $\ET(t)$, and $\EU(t)$ be events such that
\begin{itemize}
\item[(i)] if $\{\tilmu_k(t) \geq z\}$, $\ES(t)$, and $\ET(t)$ occurred
then the arm $k$ is drawn at round $t$,
\item[(ii)] $\tilmu_k(t)$, $\ES(t)$ and $\ET(t)$ are mutually independent given $\{\hat{\mu}_i(t)\}_{i=1}^K$ and $\{N_i(t)\}_{i=1}^K$. 
\item[(iii)] The event $\EU(t)$ is deterministic given $\{\hat{\mu}_i(t)\}_{i=1}^K$ and $\{N_i(t)\}_{i=1}^K$.
\item[(iv)] Given $\{\hat{\mu}_i(t)\}_{i=1}^K$ and $\{N_i(t)\}_{i=1}^K$ such that $\EU(t)$ holds, $\ET(t)$ occurs with probability at least $q>0$.
\end{itemize}
Then
\begin{multline*}
  \Expect\left[
 \sum_{t=1}^T \Ind\{\tilmu_k(t) < z, \ES(t), \EU(t), N_k(t)<N_c\}
 \right]\\
=  O\left(\frac{1}{q (\mu_k-z)^2}\right) + N_c \frac{1-q}{q}.
\end{multline*}
In particular, by setting $\ET(t)$ and $\EU(t)$ the trivial events that always hold ($q=1$), we obtain the following inequality: 
\begin{align}
  \Expect\left[
 \sum_{t=1}^T \Ind\{\tilmu_k(t) < z, \ES(t)\}
 \right]
&=  O\left(\frac{1}{(\mu_k-z)^2}\right). \label{ineq_lem_theta}
\end{align}
\end{lemma}

\begin{proof}
First we have
\begin{align}
\lefteqn{
 \sum_{t=1}^T \Ind\{\tilmu_k(t) < z, \ES(t), \EU(t), N_k(t)<N_c\} \hspace{-14em}
}\nn
 & \leq \sum_{n=0}^{N_c} \sum_{t=1}^{T} \Ind\{\tilmu_k(t) < z, \ES(t), \EU(t), N_k(t)=n\} \nonumber\\
 & \leq \sum_{n=0}^{ N_c } \sum_{m=1}^T \hspace{-0.1em}\Ind\hspace{-0.1em}\left[m \hspace{-0.1em}\leq\hspace{-0.1em} \sum_{t=1}^{T} \Ind\{ \tilmu_k(t) \hspace{-0.1em}<\hspace{-0.1em} z,
\ES(t), \EU(t), N_k(t)=n\} \right]. \nonumber\\\label{ineq:atwosum}
\end{align}
Here note that the event
\begin{equation*}
 m \leq \sum_{t=1}^{T} \Ind\{ \tilmu_k (t)< z, \ES(t), \EU(t), N_k(t)=n\}
\end{equation*}
implies that the event
\begin{align}
\{\ES(t), \EU(t), N_k(t)=n\}\label{zentei}
\end{align}
occurred
for at least $m$ rounds and
$\{\tilmu_k(t)<z\}$ or $\ET^c(t)$ occurred for the first $m$ rounds
such that \eqref{zentei} occurred.
Thus, by using the mutual independence of $\{\tilmu_k(t)<z\}$, $\ES(t)$, and $\ET(t)$, we have
\begin{multline}
\hspace{-1em}\Pr\left[m \leq \sum_{t=1}^{T} \Ind\{ \tilmu_k(t) < z,
\ES(t), \EU(t), N_k(t)=n\} \Bigg| \hatmu_{k,n} \right]\\
\le (1 - p_{k, n}(z) q)^m \label{ineq:amcont}
\end{multline}
and therefore
\begin{align*}
\lefteqn{
 \Expect\left[ \sum_{t=1}^T \Ind\{\tilmu_k(t) < z, \ES(t), \EU(t), N_k(t)<N_c\}\Bigg| \hatmu_{k,n}\right] 
}\\
&\leq \sum_{n=0}^{ N_c } \sum_{m=1}^T (1 - p_{k, n}(z)q)^m   \text{\hspace{1em}(by (\ref{ineq:atwosum}) and (\ref{ineq:amcont}))} \\
&\leq \sum_{n=0}^{N_c} \frac{1 - p_{k, n}(z) q}{p_{k, n}(z) q} 
 = \frac{1}{q} \sum_{n=0}^{T-1} \left( \frac{1}{p_{k, n}(z)} - 1 \right) + N_c \frac{1-q}{q}. 
\end{align*}
By using Lemma \ref{lem:bernoullideviatebound}, we obtain
\begin{align}
\lefteqn{
\Expect\left[\sum_{n=0}^{T-1} \left( \frac{1}{p_{k, n}(z)} - 1 \right) \right]
}\nn
& \leq \frac{24}{\Delta_k(z)^2}\nn
&+ \hspace{-1em} \sum_{n=\lceil 8/ \Delta_k(z) \rceil}^{T-1} \hspace{-1.2em} O \Biggl(e^{-\Delta_k(z)^2 n / 2}\hspace{-0.1em} + \hspace{-0.1em}\frac{e^{-D_k(z) n}}{(n+1)\Delta_k(z)^2} \hspace{-0.1em}+\hspace{-0.1em} \frac{1}{e^{\Delta_k(z)^2 n / 4}-1} \hspace{-0.1em}\Biggr)\hspace{-0.1em}. \nonumber\\ \label{ineq:agsummation}
\end{align}
By using the fact that $D_k(z) = d(z, \mu_k) = \Omega(1/(\mu_k-z)^2)$ (from the Pinsker's inequality), it is easy to verify that the RHS of (\ref{ineq:agsummation}) is $O(1/(\mu_k-z)^2)$.
By using these facts, we finally obtain 
\begin{align*}
\lefteqn{
\Expect\left[
 \sum_{t=1}^T \Ind\{\tilmu_k(t) < z, \ES(t), \EU(t), N_k(t)<N_c\}
 \right]
}\nn
& \le \frac{1}{q} \Expect \left[
\sum_{n=0}^{T-1} \left( \frac{1}{p_{k, n}(z)} - 1 \right) 
 \right] + N_c \frac{1-q}{q} \nn
&=  O\left(\frac{1}{q (\mu_k-z)^2}\right) + N_c \frac{1-q}{q},
\end{align*}
which concludes the proof of the lemma.
\end{proof}

\begin{lemma} {\rm (Deviation of empirical averages, \citet[Appendix B.1]{shiprafurther})} \label{lem_avgdeviation}
Let $k \in [K]$ and $z > \mu_k$ be arbitrary. Then,
\begin{equation*}
  \Expect\left[ \sum_{t=0}^{\infty} \Ind\{\EA_k(t), \hatmu_k(t) > z \} \right] < 1 + \frac{1}{d(z, \mu_k)}.  
\end{equation*}
\end{lemma}

\begin{lemma} {\rm (Deviation of Beta posteriors)}  \label{lem_chernoffdeviation}
Let $k \in [K]$, $x_1, x_2 \in [0,1]$ be arbitrary values such that $x_1>x_2$, and $n\geq1$. Then,
\begin{multline*}
  \Prob( \theta_k(t) \geq x_1 | \hatmu_k(t) \leq x_2, N_k(t)=n) \\ \leq \exp{\left( - d(x_2, x_1) n \right)}.
\end{multline*}
\end{lemma}
\begin{proof}
Note that, this lemma is essentially the same as the first display in \citet[Appendix B.2]{shiprafurther}. While \citet{shiprafurther} provide a bound for $N_k(t) > n$, the bound in our lemma is for $N_k(t) = n$. For the sake of rigor, we write the proof here.
\begin{align*}
\lefteqn{
  \Prob( \theta_{j}(t) \geq x_1 | \hatmu_{j}(t) \leq x_2, N_j(t)=n)
} \\
  & = \Prob\biggl( \theta \sim \Beta(\hatmu_{j}(t) n + 1, (1-\hatmu_{j}(t)) n +1),  \nonumber\\
  & \hspace{5em} \theta \geq x_1 \biggr| \hatmu_{j}(t) \leq x_2 \biggr) \nonumber\\
  & = 1 - F_{x_2 n + 1, (1-x_2) n +1}^{\mathrm{beta}}(x_1) \nonumber\\
  & = F_{n+1,x_1}^{\mathrm{B}}(x_2 n) \nn
   & \text{\hspace{8em} (by the Beta-Binomial equality)} \nonumber\\
  & \leq F_{n,x_1}^{\mathrm{B}}(x_2 n) \leq \exp{\left( - d(x_2, x_1) n \right)}\
  \\ & \text{\hspace{8em} (by the Chernoff bound)}. \nonumber\\ 
\end{align*}
\end{proof}

\subsection{Proof of Lemma \ref{lem:fourterms}}

\label{subsec:forterms}

\textbf{Evaluation of term (A):}

\begin{proof}
Here, we prove inequality (\ref{ineq:boundterma}).
Recall that 
\begin{equation*}
  \mathrm{(A)} = \sum_{t=1}^T \Ind\{\EB^c(t)\} = \sum_{t=1}^T \Ind\{\theta^*(t) < \mu_L^{(-)}\}.
\end{equation*}
Since $\theta^*(t)$ is the $L$-th largest posterior sample among arms at round $t$, $\theta^*(t) < \mu_L^{(-)}$ implies that, there exists at least one arm in $\opts$ with its posterior sample smaller than $\mu_L^{(-)}$. Namely, 
\begin{equation*}
\{\theta^*(t) < \mu_L^{(-)} \} \subset
 \bigcup_{k \in \opts} \{\theta_k(t) < \mu_L^{(-)}\},
\end{equation*}
and therefore
\begin{align*}
\lefteqn{
\{\theta^*(t) < \mu_L^{(-)} \}
}\nn
&=  
 \bigcup_{k \in \opts} \{\theta_k(t) < \mu_L^{(-)}, \theta^*(t) < \mu_L^{(-)}\}\nn
&=  
 \bigcup_{k \in \opts} \{\theta_k(t) < \mu_L^{(-)}, \dmmax{j\in[L]}{L}\theta_j(t) < \mu_L^{(-)}\}\nn
&\subset
 \bigcup_{k \in \opts} \{\theta_k(t) < \mu_L^{(-)},
{\max_{j\in [L]\setminus \{k\}}}^{\!\!\!\!\!\!(L)}
\theta_j(t) < \mu_L^{(-)}\}.
\end{align*}
By using the union bound, we obtain
\begin{align*}
\lefteqn{
\Ind\{ \theta^*(t) < \mu_L^{(-)} \} 
}\nn
&\leq \sum_{k \in \opts} \Ind
\{\theta_k(t) < \mu_L^{(-)},
{\max_{j\in [L]\setminus \{k\}}}^{\!\!\!\!\!\!(L)}
\theta_j(t) < \mu_L^{(-)}\}.
\end{align*}
Note that the event $\max_{j\in [L]\setminus\{k\}}^{(L)}\theta_j(t)<\mu_L^{(-)}$
satisfies the condition for the event $\ES(t)$ in \eqref{ineq_lem_theta} in Lemma \ref{lem_theta_gen} with $z:=\mu_L^{(-)}$.
Therefore we obtain from Lemma \ref{lem_theta_gen} that
\begin{multline*}
\Expect\left[ \sum_{t=1}^T \Ind\{ \theta^*(t) < \mu_L^{(-)} \}  \right] \\
 = O\left(\frac{1}{(\mu_k-\mu_L^{(-)})^2}\right)
 = O\left(\frac{1}{(\mu_L-\mu_L^{(-)})^2}\right),
\end{multline*}
which concludes the proof of inequality (\ref{ineq:boundterma}).
\end{proof}

\textbf{Evaluation of term (B):}

\begin{proof}
Here, we prove inequality (\ref{ineq:boundtermb}). 
We have,
\begin{align}
\lefteqn{
\mathrm{(B)} = \sum_{t=1}^T \Ind\{\EA_i(t), \EC_i^c(t)\} 
} \nn
 & = \sum_{t=1}^T \Ind\left\{ \bigcup_{j \in \lisubopts} \{ \EA_i(t), \tilmuijsst < \nu \} \right\} \nn
 & = \sum_{t=1}^T \sum_{j \in \lisubopts} \Ind\left\{ \EA_i(t), \tilmuijsst < \nu \right\} \nn
 & = \sum_{t=1}^T \sum_{j \in \lisubopts} \nn
 &
 \hspace{-0.5em}\left\{\hspace{-0.1em} \Ind\hspace{-0.1em}\left\{ \EA_i(t), \hatmu_i(t) \hspace{-0.1em}>\hspace{-0.1em} \mu_L \right\} \hspace{-0.2em}+\hspace{-0.2em} \Ind\hspace{-0.1em}\left\{ \EA_i(t), \hatmu_i(t) \hspace{-0.1em}\leq\hspace{-0.1em} \mu_L, \tilmuijsst \hspace{-0.1em}<\hspace{-0.1em} \nu \hspace{-0.1em}\right\} \hspace{-0.2em}\right\}\hspace{-0.1em}. \label{termbhatdiv}
\end{align}
In the following, we bound the first and the second terms in the inner sum of the last line of \eqref{termbhatdiv}.
From Lemma \ref{lem_avgdeviation}, the first term of \eqref{termbhatdiv} is bounded as
\begin{align*}
 \Expect\left[ \sum_{t=1}^T \Ind\left\{ \EA_i(t), \hatmu_i(t) > \mu_L \right\} \right] \le 1 + \frac{1}{d(\mu_L, \mu_i)}  = O(1).
\end{align*}

On the other hand, the second term of \eqref{termbhatdiv} is transformed as
\begin{align*}
\lefteqn{
 \sum_{t=1}^T \Ind\left\{\EA_i(t), \hatmu_i(t) \leq \mu_L, \tilmuijsst < \nu \right\} 
 }\\ 
 & \le \frac{\log{\log{T}}}{d(\mu_L, \nu)} \nn
 & \hspace{-0.3em} + \hspace{-0.1em}\sum_{t=1}^T\hspace{-0.1em} \hspace{-0.1em}\Ind\hspace{-0.1em} \left\{\EA_i(t), N_i(t) \hspace{-0.1em}>\hspace{-0.1em} \frac{\log{\log{T}}}{d(\mu_L, \nu)}, \hatmu_i(t) \hspace{-0.1em}\leq\hspace{-0.1em} \mu_L, \tilmuijsst \hspace{-0.1em}<\hspace{-0.1em} \nu \right\}
 \\ & \leq \frac{\log{\log{T}}}{d(\mu_L, \nu)} \nn
 & +\sum_{t=1}^T  \Ind\left\{ N_i(t)  > \frac{\log{\log{T}}}{d(\mu_L, \nu)},  \hatmu_i(t) \leq  \mu_L, \tilmuijsst  < \nu \right\}.
\end{align*}

Since $\tilmuijsst$ is the $(L-1)$-th largest posterior sample among arms except for  $i$ and $j$, $\tilmuijsst < \nu$ indicates that, the number of arms excluding $i$ and $j$ with posterior samples larger than or equal to $\nu$ is at most $L-2$, and thus at least one arm among $[L-1]$ has its posterior smaller than $\nu$. Namely,
\begin{align*}
\lefteqn{
\{\tilmuijsst < \nu\}
=\{{\max_{l\in [K]\setminus \{i,j\}}}^{\!\!\!\!\! (L-1)}\tilmu_l(t)<\nu\}
}\nn
&=\bigcup_{k\in[L-1]}\{\tilmu_k(t)<\nu,{\max_{l\in [K]\setminus \{i,j\}}}^{\!\!\!\!\! (L-1)}\tilmu_l(t)<\nu\}\nn
&\subset\bigcup_{k\in[L-1]}\{\tilmu_k(t)<\nu,{\max_{l\in [K]\setminus \{i,j,k\}}}^{\!\!\!\!\!\!\!\! (L-1)}\tilmu_l(t)<\nu\}.
\end{align*}
By using this, we have
\begin{align*}
\lefteqn{
  \sum_{t=1}^T \Ind\left\{N_i(t) > \frac{\log{\log{T}}}{d(\mu_L, \nu)}, \hatmu_i(t) \leq \mu_L, \tilmuijsst < \nu \right\}
} \nn
 & \leq \sum_{t=1}^T \sum_{k\in[L-1]} \Ind\Bigl\{N_i(t) > \frac{\log{\log{T}}}{d(\mu_L, \nu)}, \hatmu_i(t) \leq \mu_L,\nn
 &\hspace{6em} \tilmu_k(t)<\nu,{\max_{l\in [K]\setminus \{i,j,k\}}}^{\!\!\!\!\!\!\!\! (L-1)}\tilmu_l(t)<\nu \Bigr\}.
\end{align*}

Moreover, let $\nu_2 = (\nu+\mu_L)/2 =(\mu_{L-1}+3\mu_L)/4$. 
For $k\in[L-1]$, $\mu_k > \nu > \nu_2 > \mu_L$ and 
\begin{align*}
\lefteqn{
 \Prob\left\{\tilmu_k(t)<\nu, N_k(t) \ge \frac{\log{T}}{2(\nu-\nu_2)^2}\right\}
} \nn
& \le \sum_{n=\frac{\log{T}}{2(\nu-\nu_2)^2}}^T \Prob\{\tilmu_k(t)<\nu, N_k(t) = n \} \nn
& \le \sum_{n=\frac{\log{T}}{2(\nu-\nu_2)^2}}^T \Prob\{\tilmu_k(t)<\nu, \hatmu_k(t) > \nu_2, N_k(t) = n \} \nn
&\hspace{2em}+ \sum_{n=\frac{\log{T}}{2(\nu-\nu_2)^2}}^T \Prob\{\hatmu_k(t) \le \nu_2, N_k(t) = n \} \nn
& \le \sum_{n=\frac{\log{T}}{2(\nu-\nu_2)^2}}^T \e^{-d(\nu_2, \nu)n} \nn
&\hspace{2em}+ \sum_{n=\frac{\log{T}}{2(\nu-\nu_2)^2}}^T \Prob\{\hatmu_k(t) \le \nu_2, N_k(t) = n \} \nn
& \hspace{15em} \text{ (by Lemma \ref{lem_chernoffdeviation})}\nn
& \le \sum_{n=\frac{\log{T}}{2(\nu-\nu_2)^2}}^T \e^{-d(\nu_2, \nu)n}  + \sum_{n=\frac{\log{T}}{2(\nu-\nu_2)^2}}^T \e^{-d(\nu_2, \mu_k) n} \nn
&\hspace{15em} \text{ (by Chernoff bound)}\nn
& = O(1/T) \text{ (by $(\mu_k-\nu_2) > (\nu-\nu_2)$ and Pinsker's inequality)}
\end{align*}
and thus
\begin{align*}
\lefteqn{
\sum_{t=1}^T \sum_{k\in[L-1]} \Prob\Bigl\{N_i(t) > \frac{\log{\log{T}}}{d(\mu_L, \nu)}, 
} \nn
 & \hspace{2em}\hatmu_i(t) \leq \mu_L, \tilmu_k(t)<\nu,{\max_{l\in [K]\setminus \{i,j,k\}}}^{\!\!\!\!\!\!\!\! (L-1)}\tilmu_l(t)<\nu \Bigr\}\nn
 & \leq \hspace{-0.1em} \sum_{t=1}^T \hspace{-0.1em} \sum_{k\in[L-1]} \hspace{-0.6em} \Prob\Bigl\{N_i(t) \hspace{-0.1em}>\hspace{-0.1em} \frac{\log{\log{T}}}{d(\mu_L, \nu)}, N_k(t) \hspace{-0.1em}<\hspace{-0.1em} \frac{\log{T}}{2(\nu-\nu_2)^2},\nn
 & \hspace{2em}\hatmu_i(t) \hspace{-0.1em}\leq\hspace{-0.1em} \mu_L, \tilmu_k(t)\hspace{-0.1em}<\hspace{-0.1em}\nu,\hspace{-0.1em}{\max_{l\in [K]\setminus \{i,j,k\}}}^{\!\!\!\!\!\!\!\! (L-1)}\tilmu_l(t)\hspace{-0.1em}<\hspace{-0.1em}\nu \Bigr\} \nn & \hspace{2em} + \sum_{t=1}^T \sum_{k\in[L-1]} \Prob\Bigl\{ \tilmu_k(t)<\nu, N_k(t) \ge \frac{\log{T}}{2(\nu-\nu_2)^2} \Bigr\}\nn
 & \leq \hspace{-0.1em} \sum_{t=1}^T \hspace{-0.1em} \sum_{k\in[L-1]} \hspace{-0.6em} \Prob\Bigl\{N_i(t) \hspace{-0.1em}>\hspace{-0.1em} \frac{\log{\log{T}}}{d(\mu_L, \nu)}, N_k(t) \hspace{-0.1em}<\hspace{-0.1em} \frac{\log{T}}{2(\nu-\nu_2)^2},\nn
 & \hspace{2em} \hatmu_i(t) \hspace{-0.1em}\leq\hspace{-0.1em} \mu_L, \tilmu_k(t)\hspace{-0.1em}<\hspace{-0.1em}\nu,\hspace{-0.1em}{\max_{l\in [K]\setminus \{i,j,k\}}}^{\!\!\!\!\!\!\!\! (L-1)}\tilmu_l(t)\hspace{-0.1em}<\hspace{-0.1em}\nu \Bigr\} \nn & \hspace{2em} + O(1).
\end{align*}

Here,
$z:=\nu$, $\ES(t):=\{\max_{l \in [K]\setminus \{i,j,k\}}^{(L-1)}\tilmu_l(t) <\nu\}$, $\ET(t) :=\{\tilmu_i(t) \le \nu\}$, and $\EU(t) := \{\hatmu_i(t) \leq \mu_L\}$
satisfy the conditions in Lemma \ref{lem_theta_gen}. Under $\EU(t)$, $\ET(t)$ holds with probability at least
\begin{equation*}
 1 - \exp{\left( - d(\mu_L, \nu) \left(\frac{\log{\log{T}}}{d(\mu_L, \nu)}\right) \right)} = 1 - (\log{T})^{-1}
\end{equation*}
by Lemma \ref{lem_chernoffdeviation}.
Therefore, by using Lemma \ref{lem_theta_gen} with $N_c = \log{T}/(2(\nu-\nu_2)^2)$, we obtain
\begin{align}
\lefteqn{
 \Expect\Biggl[ \sum_{t=1}^T \Ind\Bigl\{N_i(t) > \frac{\log{\log{T}}}{d(\mu_L, \nu)}, N_k(t)<\frac{\log{T}}{2(\nu-\nu_2)^2},
}\nn
& \hatmu_i(t) \leq \mu_L, 
 \tilmu_k(t)<\nu,{\max_{l\in [K]\setminus \{i,j,k\}}}^{\!\!\!\!\!\!\!\! (L-1)}\tilmu_l(t)<\nu  \Bigr\} \Biggr] 
 \nn
 & \le  O\left(\frac{1}{(1 - (\log{T})^{-1}) (\mu_k-\nu)^2}\right) +\nn
 & \hspace{4em} O\left(\frac{(\log{T})^{-1}}{1-(\log{T})^{-1}} \frac{\log{T}}{2(\nu-\nu_2)^2} \right) = O(1). \label{bend}
\end{align}
From \eqref{bend} and the union bound over $k \in [L-1]$, the second term of \eqref{termbhatdiv} is $O(1)$.
In summary, term (B) is $O(\log{\log{T}})$ in expectation.
\end{proof}

\textbf{Evaluation of term (C):}

\begin{proof}
Here, we prove inequality (\ref{ineq:boundtermc}). Recall that, 
\begin{equation*}
\mathrm{(C)} = \sum_{j \in \lisubopts} \sum_{t=1}^T \Ind\{\EA_i(t), \EA_j(t), \EC_i(t), \ED_i(t)\}. 
\end{equation*}

Let $\nu_2 = (\nu+\mu_L)/2 $ $=(\mu_{L-1}+3\mu_L)/4$. Note that, we defined $\nu$ and $\nu_2$ such that $\mu_{L-1} > \nu > \nu_2 > \mu_L$, $O(\mu_{L-1} - \nu) = O(\nu - \nu_2) = O(\nu_2 - \mu_L) = O(\mu_{L-1}-\mu_L) = O(1)$ as a function of $T$.
Then,
\begin{align}
\lefteqn{
\sum_{t=1}^T \Ind\{\EA_i(t), \EA_j(t), \EC_i(t), \ED_i(t)\}
} \nn
 & = \sum_{t=1}^T \Ind\{\EA_i(t), \EA_j(t), \EC_i(t), \ED_i(t), \hatmu_j(t) > \nu_2 \} \nn
 & \hspace{2em} +  \sum_{t=1}^T \Ind\{\EA_i(t), \EA_j(t), \EC_i(t), \ED_i(t), \hatmu_j(t) \leq \nu_2 \}  \nn
 & \leq \sum_{t=1}^T \Ind\{\EA_j(t), \hatmu_j(t) > \nu_2 \}\nn
 & \hspace{2em} + \sum_{t=1}^T \Ind\{\EA_i(t), \EA_j(t), \EC_i(t), \ED_i(t), \hatmu_j(t) \leq \nu_2\}.  \nn \label{ineq:c1andc2}
\end{align}

By using Lemma \ref{lem_avgdeviation} with $z:=\nu_2$, the first term in \eqref{ineq:c1andc2} is bounded as:
\begin{multline}
 \Expect\left[ \sum_{t=1}^T \Ind\{\EA_j(t), \hatmu_j(t) > \nu_2 \} \right] \leq 1 + \frac{1}{d(\nu_2, \mu_j)} \\ = O\left(\frac{1}{(\nu_2 - \mu_j)^2}\right) = O\left(\frac{1}{(\mu_{L-1}- \mu_L)^2}\right) = O(1).
\label{ineq:c2boundfst}
\end{multline}

We now bound the second term in \eqref{ineq:c1andc2}.
Let $\EC'_{i,j}(t) = \{ \thetaijsst \geq \nu \} \supset \EC_i(t)$.
Let $\EE_j(t) = \{ N_j(t) \geq \epsilon_2 \log{T} \}$. We have,
\begin{align*}
\lefteqn{
 \sum_{t=1}^T \Ind\{\EA_i(t), \EA_j(t), \EC_i(t), \ED_i(t), \hatmu_j(t) \leq \nu_2\}
} \nn
 & \le \sum_{t=1}^T \Ind\{\EA_i(t), \EA_j(t), \EC'_{i,j}(t), \ED_i(t), \hatmu_j(t) \leq \nu_2\} \nn
 & \le \epsilon_2 \log{T} \nn
 & \hspace{0.5em} + \sum_{t=1}^T \Ind\{\EA_i(t), \EA_j(t), \EC'_{i,j}(t), \ED_i(t), \hatmu_j(t) \leq \nu_2, \EE_j(t)\}. \\
 & \leq \epsilon_2 \log{T} + \sum_{n=0}^{N_i^{\mathrm{suf}}(T)-1} \sum_{t=1}^T 
 \nn & \hspace{2em} \Ind\{\EA_i(t), \EA_j(t), \EC'_{i,j}(t), N_i(t)=n, \hatmu_j(t) \leq \nu_2, \EE_j(t)\}.
\end{align*}

In the following, we bound 
\begin{equation}
  \sum_{t=1}^T \Ind\{\EA_i(t), \EA_j(t), \EC'_{i,j}(t), N_i(t)=n, \hatmu_j(t) \leq \nu_2, \EE_j(t)\}. \label{simuldraw}
\end{equation}
Note that, \eqref{simuldraw} is at most $1$ since $\{\EA_i(t), N_i(t)=n\}$ occurs at most once.
Let $\tau$ be the first round (if exists) at which $\{\EC'_{i,j}(t), \thetaijsst \leq \theta_i(t),  \EA_i(t), N_i(t)=n\}$ is satisfied. It is necessary that $\{\theta_j(\tau) \geq \thetaijss(\tau)\}$ for $\eqref{simuldraw}$ to be $1$: this is because, (i) both $\theta_i(\tau)$ and $\theta_j(\tau)$ need to be larger than $\thetaijss(\tau)$ for the simultaneous draw of arms $i$ and $j$, (ii) and if $\theta_j(\tau) < \thetaijss(\tau)$ then arm $i$ is drawn and thus $\{N_i(t)=n\}$ is never satisfied after $t>\tau$. Here,
\begin{multline*}
  \Prob \{\theta_j(\tau) \ge \thetaijss(\tau), \thetaijss(\tau) \ge \nu, \hatmu_j(\tau) \leq \nu_2\} \\ \leq \exp{\left( - d(\nu_2, \nu) N_j(\tau) \right)},
\end{multline*}
by Lemma \ref{lem_chernoffdeviation}. Therefore, we have
\begin{multline}
 \Expect\left[ \sum_{t=1}^T \Ind\{\EA_i(t), \EA_j(t), \EC_i(t), N_i(t)=n, \hatmu_j(t) \leq \nu_2\} \right]
\\
 \leq \exp{\left( - d(\nu_2, \nu) \epsilon_2 \log{T}  \right)} = T^{- \epsilon_2 d(\nu_2, \nu)}. 
\label{ineq:individualc2bound} 
\end{multline}

In summary, the second term in \eqref{ineq:c1andc2} is bounded as:
\begin{align*}
\lefteqn{
\Expect\Biggl[ \sum_{t=1}^T \Ind\{\EA_i(t), \EA_j(t), \EC_i(t), \ED_i(t), \hatmu_j(t) \leq \nu_2\} \Biggr] 
} \nn
  & \leq \epsilon_2 \log{T} + N_i^{\mathrm{suf}}(T) T ^ {- \epsilon_2 d(\nu_2, \nu) } \nonumber\\
  & \leq \left( \epsilon_2 + \frac{4 T ^ {- \epsilon_2 d(\nu_2, \nu) } }{d(\mu_i, \mu_L)} \right) \log{T} \text{\hspace{2em}(by $(1+\epsilononed)^2 < 4$)}, 
\end{align*}
and thus,
\begin{align*}
\lefteqn{
 \Expect[\mathrm{(C)}] 
} \nn
 & \leq \hspace{-2em} \sum_{j \in \lisubopts} \left( \frac{\left(\epsilon_2 +  4 T ^ {- \epsilon_2 d(\nu_2, \nu) }\right) \log{T}}{d(\mu_i, \mu_L)} \right) + O(1) \nn
 & \leq \hspace{-2em} \sum_{j \in \lisubopts} \left( \frac{\left(\epsilon_2 +  4 T ^ {- \epsilon_2 \Delta_{L,L-1}^2/8 }\right) \log{T}}{d(\mu_i, \mu_L)} \right) + O(1),
\end{align*}
where we used the fact that $d(\nu_2, \nu) \geq 2(\nu-\nu_2)^2 = 2 \times ((\mu_{L-1}-\mu_L)/4)^2$ in the last transformation.
\end{proof}

\textbf{Evaluation of term (D):}

\begin{proof}
Here, we prove inequality (\ref{ineq:boundtermd}). We first divide term (D) into two subterms as:
\begin{align}
\lefteqn{
\Expect[\mathrm{(D)}] = 
 \Expect\left[ \sum_{t=1}^T \Ind\{\EA_i(t), \EB(t), N_i(t) \geq N_i^{\mathrm{suf}}(T)\} \right]
} \nn 
 & \leq \Expect\left[ \sum_{t=1}^T \Ind\{\EA_i(t), \EB(t), \hatmu_i(t) > \mu_i^{(+)}, N_i(t) \geq N_i^{\mathrm{suf}}(T)\} \right]\nn
 &  + \Expect\left[\sum_{t=1}^T \Ind\{\EA_i(t), \EB(t), \hatmu_i(t) \leq \mu_i^{(+)}, N_i(t) \geq N_i^{\mathrm{suf}}(T)\} \right]. \label{ineq:dtwoterms}
\end{align}

On one hand, the first term in \eqref{ineq:dtwoterms} is bounded as:
\begin{align}
\lefteqn{
 \Expect\left[ \sum_{t=1}^T \hspace{-0.2em} \Ind\{\EA_i(t), \EB(t), \hatmu_i(t) > \mu_i^{(+)}\hspace{-0.5em}, N_i(t) \geq N_i^{\mathrm{suf}}(T)\} \hspace{-0.2em} \right] 
} \nonumber\\
 &  \leq \Expect\left[ \sum_{t=1}^T \Ind\{\EA_i(t), \hatmu_i(t) > \mu_i^{(+)}\} \right] \nonumber\\
  &  \leq 1 + \frac{1}{d(\mu_i^{(+)}, \mu_i)} \text{\hspace{3em} (by Lemma \ref{lem_avgdeviation}).}
\end{align}

On the other hand, each component of the second term of \eqref{ineq:dtwoterms} is bounded as
\begin{align}
\lefteqn{
\Expect\left[
\Ind[\EA_i(t), \EB(t), \hatmu_i(t) \leq \mu_i^{(+)}, N_i(t) \geq N_i^{\mathrm{suf}}(T)]\right]
} \nonumber\\
 & \leq \Expect\left[\Ind[\theta_{i}(t) \geq \mu_L^{(-)}, \hatmu_i(t) \leq \mu_i^{(+)}, N_i(t) \geq N_i^{\mathrm{suf}}(T)]\right] \nonumber\\
 & = \Expect\Big[\Expect\big[\Ind[\theta_{i}(t) \geq \mu_L^{(-)}, \hatmu_i(t) \leq \mu_i^{(+)}, N_i(t) \geq N_i^{\mathrm{suf}}(T)]\nonumber\\
&\phantom{wwwwwwwwwwwwwwwwwwwwwww}
\big| \hatmu_i(t),\,N_i(t)\big]\Big] \nonumber\\
 & \le \Expect\Big[\Expect\big[\Ind[\hatmu_i(t) \leq \mu_i^{(+)}, N_i(t) \geq N_i^{\mathrm{suf}}(T)]
\nonumber\\
&\phantom{wwwwwww}
\Prob[\theta_{i}(t) \geq \mu_L^{(-)}|\hatmu_i(t),\,N_i(t)]
\big| \hatmu_i(t),\,N_i(t)\big]\Big] \nonumber\\
& \le \Expect\left[\Expect\left[
\exp(-d(\mu_i^{(+)}, \mu_L^{(-)})N_i^{\mathrm{suf}}(T))
\Big| \hatmu_i(t),\,N_i(t)\right]\right] \nonumber\\
& \phantom{wwwwwwwwwwwwwwwww}
\text{\hspace{1em} (by Lemma \ref{lem_chernoffdeviation})} \nn
 & =\exp{( - d(\mu_i^{(+)}, \mu_L^{(-)}) N_i^{\mathrm{suf}}(T) )}\nn
 & =T^{-1} \phantom{wwwwwwww}\mbox{(by the definition of $N_i^{\mathrm{suf}}(T)$),}
\label{d2}
\end{align}
where we used the fact $\Expect[X]=\Expect[\Expect[X|Y]]$
for any random variables $X$ and $Y$.
Putting \eqref{ineq:dtwoterms}--\eqref{d2} together
we obtain
\begin{equation}
 \Expect[\mathrm{(D)}] \leq 1 + \frac{1}{d(\mu_i^{(+)}, \mu_i)} +
\sum_{t=1}^T T^{-1},
\end{equation}
from which the inequality (\ref{ineq:boundtermd}) follows.
\end{proof}

\subsection{Proof of Lemma \ref{lem:termcconvergence}}
\label{subsec:proofe2lim}

It suffices to prove that
for any $a,b>0$
\begin{equation*}
  \inf_{\epsilon_2 > 0} \left\{ \frac{ T^{-a \epsilon_2} }{ b } + \epsilon_2 \right\} = O\left(\frac{\log \log T}{\log T}\right).
\end{equation*}

By letting $\epsilon_2=(\log \log T)/(a\log T)$, we have
\begin{align*}
 \inf_{\epsilon_2 > 0} \left\{ \frac{ T^{-a \epsilon_2} }{ b } + \epsilon_2 \right\}
&=
 \inf_{\epsilon_2 > 0} \left\{ \frac{ e^{-a \epsilon_2\log T} }{ b } + \epsilon_2 \right\}
\nn
&\le
\frac{e^{-\log\log T}}{ b } + \frac{\log \log T}{a\log T}\nn
&=
\frac{1}{b\log T}+\frac{\log \log T}{a\log T}\nn
&=
O\left(\frac{\log \log T}{\log T}\right)
\end{align*}
and the proof is completed.

\end{document}